\documentclass[12pt,draftcls,onecolumn]{IEEEtran}

\usepackage{amsmath}
\usepackage{amssymb}  
\usepackage{amsfonts}
\usepackage{mathtools}
\usepackage{bm}
\usepackage{color}
\usepackage{graphics} 
\usepackage{graphicx}
\usepackage{epsfig} 
\usepackage{epstopdf}
\usepackage[noadjust]{cite}
\usepackage{tikz}
\usetikzlibrary{calc,positioning,shapes,shadows,arrows,fit}
\usepackage{fancyhdr}
\usepackage{fancyref}
\usepackage{hyperref}
\usepackage{float}
\usepackage[font = footnotesize]{caption}

\newcommand{\scr}{\scriptscriptstyle }

\usepackage{amsthm}

\newtheorem{theorem}{Theorem}
\newtheorem{lemma}{Lemma}

\theoremstyle{definition}

\newtheorem{definition}{Definition}
\newtheorem{remark}{Remark}
\newtheorem{assumption}{Assumption}

\newtheorem{problem}{Problem}

\begin{document}
	
	\title{Communication-based Decentralized Cooperative Object Transportation Using Nonlinear Model Predictive Control}
	
	\author{Christos K. Verginis, Alexandros Nikou and Dimos V. Dimarogonas%\IEEEmembership{Member,~IEEE}% <-this % stops a space
		\thanks{The authors are with the ACCESS Linnaeus Center, School of Electrical Engineering, KTH Royal Institute of Technology, SE-100 44, Stockholm, Sweden and with the KTH Center for Autonomous Systems. Email: {\tt\small \{cverginis, anikou, dimos\}@kth.se}. This work was supported by the H2020 ERC Starting Grant BUCOPHSYS, the EU H2020 Co4Robots Project, the EU H2020 AEROWORKS project, the swedish Foundation for Strategic Research (SSF), the Swedish Research Council (VR) and the Knut och Alice Wallenberg Foundation (KAW).}}% <-this % stops a space
	% <-this % stops a space
	
	
	
	\maketitle
	
	\begin{abstract}
This paper addresses the problem of cooperative transportation of an object rigidly grasped by $N$ robotic agents. We propose a Nonlinear Model Predictive Control (NMPC) scheme that guarantees the navigation of the object to a desired pose in a bounded workspace with obstacles, while complying with certain input saturations of the agents. The control scheme is based on inter-agent communication and is decentralized in the sense that each agent calculates its own control signal. Moreover, the proposed methodology ensures that the agents do not collide with each other or with the workspace obstacles as well as that they do not pass through singular configurations. The feasibility and convergence analysis of the NMPC are explicitly provided. Finally, simulation results illustrate the validity and efficiency of the proposed method.
	\end{abstract}
	
	
	% \fi
	\IEEEpeerreviewmaketitle
	
	\section{Introduction}

Over the last years, multi-agent systems have gained a significant amount of attention, due to the advantages they offer with respect to single-agent setups. Robotic manipulation is a field where the multi-agent formulation can play a critical role, since a single robot might not be able to perform manipulation tasks that involve heavy payloads and challenging maneuvers.

Regarding cooperative manipulation, the literature is rich with works that employ control architectures where the robotic agents communicate and share information with each other as well as completely decentralized schemes, where each agent uses only local information or observers \cite{schneider1992object,liu1998decentralized,zribi1992adaptive,khatib1996decentralized,caccavale2000task,gudino2004control}. The most common methodology used in the related literature constitutes of impedance and force/motion control \cite{schneider1992object,caccavale2008six,heck2013internal,erhart2013adaptive,szewczyk2002planning,tsiamis2015cooperative,ficuciello2014cartesian,ponce2016cooperative,gueaieb2007robust}. Most of the aforementioned works employ force/torque sensors to acquire knowledge of the manipulator-object contact forces/torques, which, however, may result to performance decline due to sensor noise. 

Moreover, in manipulation tasks, such as pose/force or trajectory tracking, collision with obstacles of the environment has been dealt with only by exploiting the potential extra degrees of freedom of over-actuated agents, or by using potential field-based algorithms. These methodologies, however, may suffer from local minima, even in single-agent cases, and in many cases they yield high control inputs that do not comply with the saturation of actual motor inputs, especially close to collision configurations. In our previous works, \cite{verginis2017distributed,verginisMastellaro}, we considered the problem of trajectory tracking for decentralized robust cooperative manipulation, without taking into account singularity- or collision avoidance.

Another important property that concerns robotic manipulators is the singularities of the Jacobian matrix, which maps the joint velocities of the agent to a $6$D vector of generalized velocities. Such \textit{singular} \textit{kinematic} configurations, that indicate directions towards which the agent cannot move, must be always avoided, especially when dealing with task-space control in the end-effector \cite{siciliano}. In the same vein, \textit{representation} singularities can also occur in the mapping from coordinate rates to angular velocities of a rigid body. 
%Singularity avoidance is an important issue that should be taken into consideration during the control design for cooperation tasks, since the matrices that are involving in the system model need to remain non-singular at all times.

The main contribution of this work is to provide decentralized feedback control laws that guarantee the cooperative manipulation of an object in a bounded workspace with obstacles. In particular, given $N$ agents that rigidly grasp an object, we design decentralized control inputs for the navigation of the object to a final pose, while avoiding inter-agent collisions as well as collisions with obstacles. Moreover, we take into account constraints that emanate from control input saturation as well kinematic and representation singularities. The proposed approach to address this problem is the repeated solution of a Finite-Horizon Open-loop Optimal Control Problem (FHOCP) of each agent, by assigning a set of priorities. Control approaches using this strategy are referred to as Nonlinear Model Predictive Control (NMPC) (see e.g. \cite{Mayne2000789, morrari_npmpc, cannon_2001_nmpc,  frank_2003_nmpc_bible, frank_1998_quasi_infinite, frank_2003_towards_sampled-data-nmpc, grune_2011_nonlinear_mpc, camacho_2007_nmpc, borrelli_2013_nmpc, fontes_2001_nmpc_stability, alex_chris_med_2017, alex_cdc_2017_timed_abstractions, alex_ECC_2018}). A decentralized NMPC scheme has been considered in our submitted work \cite{alex_IJC_2017}, which concerns multi-agent navigation with inter-agent connectivity maintenance and collision avoidance.

In our previous work \cite{alex_chris_med_2017}, a similar problem was considered in a centralized way. However, the computation burden is high, due to the fact that the number of states in the centralized case increases proportionally with the number of agents, causing exponential increase in the computational time and memory. In this work, we decouple the dynamic model among the object and the agents by using certain load-sharing coefficients and consider a communication-based leader agent formulation, where a leader agent determines the followed trajectory for the object and the follower agents comply with it through appropriate constraints. To the best of the authors' knowledge, this is the first time that the problem of decentralized object transportation with singularity, obstacle and collision avoidance is addressed. 

The remainder of the paper is structured as follows. Section \ref{sec:preliminaries} provides preliminary background. The system dynamics and the formal problem statement are given in Section \ref{sec:Problem-Formulation}. Section \ref{sec:solution} discusses the technical details of the solution and Section \ref{sec:simulation_results} is devoted to a simulation example. Conclusions and future work are discussed in Section \ref{sec:conclusions}.

\section{Notation and Preliminaries} \label{sec:preliminaries}

The set of positive integers is denoted as $\mathbb{N}$ and the real $n$-coordinate space, with $n\in\mathbb{N}$, as $\mathbb{R}^n$;
$\mathbb{R}^n_{\geq 0}$ and $\mathbb{R}^n_{> 0}$ are the sets of real $n$-vectors with all elements nonnegative and positive, respectively. The notation $\mathbb{R}^{n\times n}_{\geq 0}$ and $\mathbb{R}^{n\times n}_{> 0}$, with $n\in\mathbb{N}$, stands for positive semi-definite and positive definite matrices, respectively. The notation $\|x\|$ is used for the Euclidean norm of a vector $x \in \mathbb{R}^n$ and $\|A\| = \max \{ \|A x \|: \|x\| = 1 \}$ for the induced norm of a matrix $A \in \mathbb{R}^{m \times n}$. Given a real symmetric matrix $A$, $\lambda_{\text{min}}(A)$
	and $\lambda_{\text{max}}(A)$ denote the minimum and the maximum absolute value of eigenvalues of $A$, respectively.
	Its minimum and maximum singular values are denoted by $\sigma_{\text{min}}(A)$ and $\sigma_{\text{max}}(A)$ respectively; $I_n \in \mathbb{R}^{n \times n}$ and $0_{m \times n} \in \mathbb{R}^{m \times n}$ are the identity matrix and the $m \times n$ matrix with all entries zeros,
	respectively. Given a set $S$, we denote by $|S|$ its cardinality and by $S^N = S \times \dots \times S$ its $N$-fold Cartesian product.

The vector connecting the origins of coordinate frames $\{A\}$ and $\{B$\} expressed in frame $\{C\}$ coordinates in $3$-D space is denoted as $p^{\scriptscriptstyle C}_{\scriptscriptstyle B/A} = [x_{\scriptscriptstyle B/A}, y_{\scriptscriptstyle B/A}, z_{\scriptscriptstyle B/A}]^\top \in\mathbb{R}^3$. Given a vector $a\in\mathbb{R}^3, S(a)$ is the skew-symmetric matrix defined according to $S(a)b = a\times b$. We further denote by $\eta_{\scriptscriptstyle A/B} = [\phi_{\scriptscriptstyle A/B}, \theta_{\scriptscriptstyle A/B}, \psi_{\scriptscriptstyle A/B}]^\top\in\mathbb{T}^3\subseteq \mathbb{R}^3$ the $x$-$y$-$z$ Euler angles representing the orientation of frame $\{A\}$ with respect to frame $\{B\}$, where $\mathbb{T}\coloneqq (-\pi,\pi)\times(-\tfrac{\pi}{2},\tfrac{\pi}{2})\times(-\pi,\pi)$; Moreover, $R^{\scriptscriptstyle B}_{\scriptscriptstyle A}\in SO(3)$ is the rotation matrix associated with the same orientation and $SO(3)$ is the $3$-D rotation group. The angular velocity of frame $\{B\}$ with respect to $\{A\}$, expressed in frame $\{C\}$ coordinates, is  denoted as $\omega^{\scriptscriptstyle C}_{{\scriptscriptstyle B/A}}\in \mathbb{R}^{3}$ and it holds that $\dot{R}^{\scriptscriptstyle B}_{\scriptscriptstyle A} = S(\omega^{\scriptscriptstyle A}_{\scriptscriptstyle B/A})R^{\scriptscriptstyle B}_{\scriptscriptstyle A}$ \cite{siciliano}. Define also the sets $\mathbb{M} = \mathbb{R}^3\times\mathbb{T}^3$, $\mathcal{N} = \{1,\dots,N\}$. We define also the set $$\mathcal{O}_z \coloneqq \mathcal{O}(c_z, \beta_{1, z},\beta_{2, z}, \beta_{3, z}) \notag  = \left\{ p \in \mathbb{R}^3 : (p-c_z)^\top P (p-c_z) \le 1 \right\},$$ as the set of an \emph{ellipsoid} in 3D, where $c_z \in \mathbb{R}^3$ is the center of the ellipsoid, $\beta_{1, z}, \beta_{2, z}, \beta_{3, z} \in \mathbb{R}_{> 0}$ the lengths of its three semi-axes and $z \ge 1$ is an index term. The eigenvectors of matrix $P \in \mathbb{R}^{3 \times 3}$ define the principal axes of the ellipsoid, and the eigenvalues of $P$ are: $\beta_{1, z}^{-2}, \beta_{2, z}^{-2}$ and $\beta_{3, z}^{-2}$. For notational brevity, when a coordinate frame corresponds to an inertial frame of reference $\left\{I\right\}$, we will omit its explicit notation (e.g., $p_{\scriptscriptstyle B} = p^{\scriptscriptstyle I}_{\scriptscriptstyle B/I}, \omega_{\scriptscriptstyle B} = \omega^{\scriptscriptstyle I}_{\scriptscriptstyle B/I}, R_{\scriptscriptstyle A} = R^{\scriptscriptstyle I}_{\scriptscriptstyle A} $). Finally, all vector and matrix differentiations will be with respect to an inertial frame $\{I\}$, unless otherwise stated.

\begin{definition} \label{def:k_class} \cite{khalil_nonlinear_systems} A continuous function $\xi : [0, a) \to \mathbb{R}_{\ge 0} $ belongs to \emph{class $\mathcal{K}$} if it is strictly increasing and $\xi(0) = 0$. If $a = \infty$ and $\lim\limits_{r \to \infty} \xi(r) = \infty$, then function $\xi$ belongs to class $\mathcal{K}_{\infty}$.
\end{definition}

\begin{lemma} \label{lemma:barbalat} \cite{michalska_1994_barbalat} Let $f$ be a continuous, positive-definite function, and $x$ be an absolutely continuous function in $\mathbb{R}$. Suppose that: $\|x(\cdot)\| < \infty$, $\|\dot{x}(\cdot)\| < \infty$ and $\lim\limits_{t \to \infty} \int\limits_0^t f\big(x(s)\big)ds < \infty$. Then, it holds: $\lim\limits_{t \to \infty} \|x(t)\| = 0$.
\end{lemma}

\section{Problem Formulation}
\label{sec:Problem-Formulation}
The formulation we adopt in this paper follows from the one from our previous work \cite{alex_chris_med_2017}. Consider a bounded and convex workspace $\mathcal{W} \coloneqq \mathcal{B}(0_{3 \times 1},r_w) \subseteq \mathbb{R}^{3}$, where $r_w > 0$ is the workspace radius, consisting of $N$ robotic agents rigidly grasping an object, as shown in Fig. \ref{fig:Two-robotic-arms}, and $Z$ static obstacles described by the ellipsoids $\mathcal{O}_z, z \in \mathcal{Z} \coloneqq \{1, \dots, Z\}$. The free space is denoted as $\mathcal{W}_{\text{free}} \coloneqq \mathcal{W}\backslash \bigcup_{z \in \mathcal{Z}} \mathcal{O}_z$. The agents are considered to be fully actuated and they consist of a base that is able to move around the workspace (e.g., mobile or aerial vehicle) and a robotic arm.
The reference frames corresponding to the $i$-th end-effector and the object's center of
mass are denoted with $\left\{ E_{i}\right\} $ and $\left\{ O\right\} $,
respectively, whereas $\left\{ I\right\} $ corresponds to an inertial
reference frame. The rigidity of the grasps implies that the agents can exert any forces/torques along every direction to the object. 
We consider that each agent $i$ knows the position and velocity only of its own state as well as its own and the object's geometric parameters. Moreover, no interaction force/torque measurements or on-line communication is required.

\begin{figure}
	\centering
	\includegraphics[width = 0.60\textwidth]{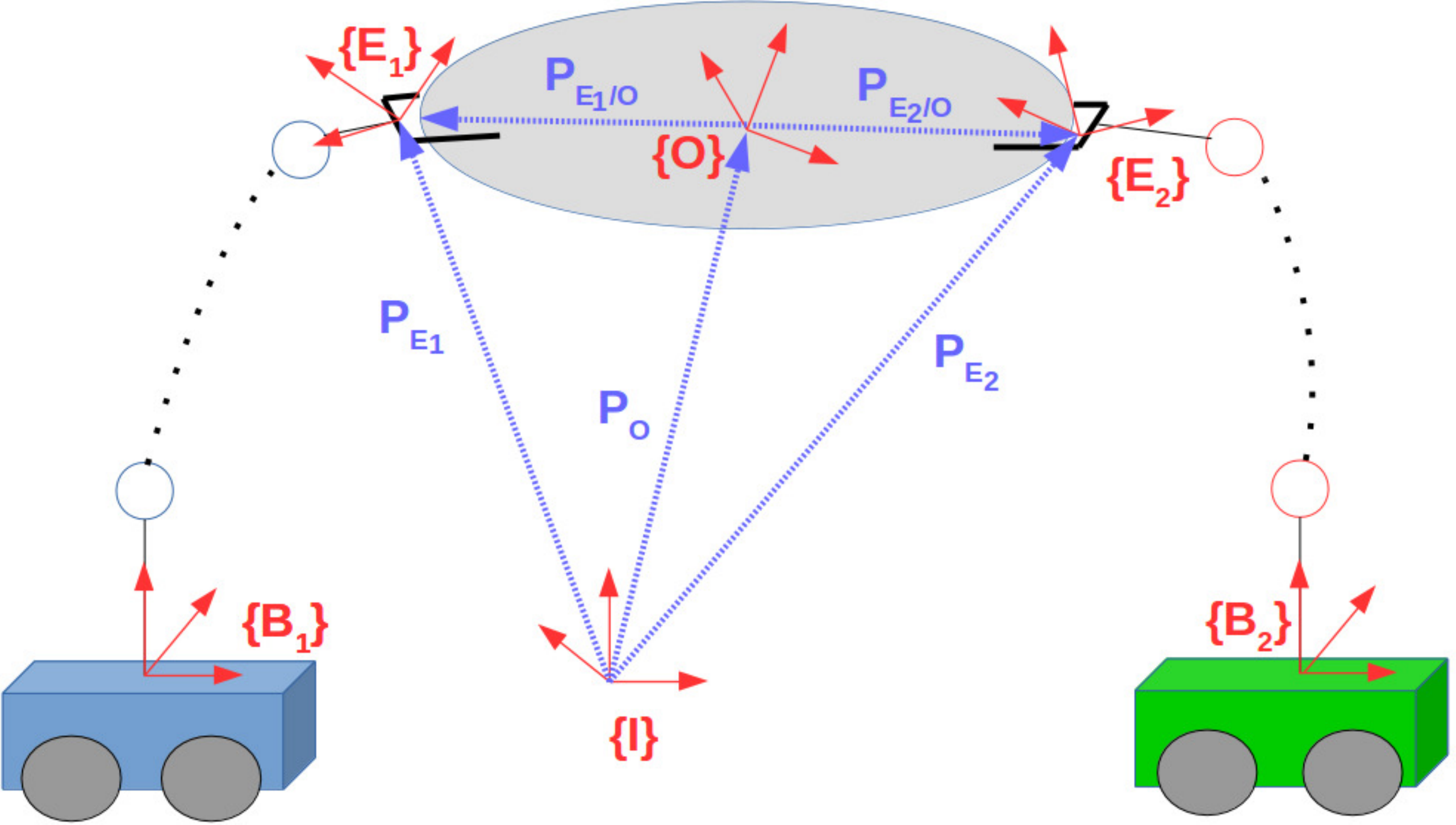}
	\caption{Two robotic arms rigidly grasping an object with the corresponding frames.\label{fig:Two-robotic-arms}}
\end{figure} 

\subsection{System model}
\label{subsec:system_model}

\subsubsection{Robotic Agents} \label{subsubsec: agent dynamics}

We denote by $q_i\in\mathbb{R}^{n_i}$ the joint space variables of agent $i\in\mathcal{N}$, with $n_i = n_{\alpha_i} + 6$, $q_i = [p^\top_{\scriptscriptstyle B_i},\eta^\top_{\scriptscriptstyle B_i}, \alpha^\top_i]^\top$, where $p_{\scriptscriptstyle B_i} = [x_{\scriptscriptstyle B_i}, y_{\scriptscriptstyle B_i}, z_{\scriptscriptstyle B_i}]^\top\in\mathbb{R}^{3},\eta_{\scriptscriptstyle B_i} = [\phi_{\scriptscriptstyle B_i}, \theta_{\scriptscriptstyle B_i}, \psi_{\scriptscriptstyle B_i}]^\top\in\mathbb{T}$ is the position and Euler-angle orientation of the agent's base, and $\alpha_i\in\mathbb{R}^{n_{\alpha_i}},n_{\alpha_i}>0$, are the degrees of freedom of the robotic arm. The overall joint space configuration vector is denoted as $q \coloneqq [q^\top_1,\dots,q^\top_N]^\top\in\mathbb{R}^n$, with $n \coloneqq \sum_{i\in\mathcal{N}}n_i$. 
%We also consider that the agents are assigned with specific priority coefficients $\delta_i \in \mathbb{N}, \forall i\in\mathcal{N}$, with $\delta_i\neq \delta_j, \forall i,j\in\mathcal{N}, i\neq j$. Without loss of generality, we assume that larger $\delta_i$ corresponds to higher priority.

The linear and angular velocities of the agents' base are described by the functions $v_{\scriptscriptstyle L,B_i}:\mathbb{R}^{n_i}\to\mathbb{R}^3$, with  $v_{\scriptscriptstyle L,B_i}(\dot{q}_i) \coloneqq \dot{p}_{\scriptscriptstyle B_i}$ and $\omega_{\scriptscriptstyle B_i}:\mathbb{R}^{2n_i}\to\mathbb{R}^3$, with $\omega_{\scriptscriptstyle B_i}(q_i,\dot{q}_i) \coloneqq J_{\scriptscriptstyle B_i}(\eta_{\scriptscriptstyle B_i})\dot{\eta}_{\scriptscriptstyle B_i}$, where $J_{\scriptscriptstyle B_i}:\mathbb{T}\to\mathbb{R}^{3\times3}$ is the representation Jacobian matrix, with: $$J_{\scriptscriptstyle B_i}(\eta_{\scriptscriptstyle B_i}) = \begin{bmatrix}
1 & 0 & \sin(\theta_{\scriptscriptstyle B_i}) \\
0 & \cos(\phi_{\scriptscriptstyle B_i}) & -\cos(\theta_{\scriptscriptstyle B_i})\sin(\phi_{\scriptscriptstyle B_i}) \\
0 & \sin(\phi_{\scriptscriptstyle B_i}) & \cos(\theta_{\scriptscriptstyle B_i})\cos(\phi_{\scriptscriptstyle B_i})
\end{bmatrix}.$$ We consider that each agent $i\in\mathcal{N}$ has access to its own state $q_i, \dot{q}_i$, and can compute, therefore the terms $v_{\scriptscriptstyle L,B_i}(\dot{q}_i), \omega_{\scriptscriptstyle B_i}(q_i,\dot{q}_i)$.

In addition, we denote as $p_{\scriptscriptstyle E_i}:\mathbb{R}^{n_i}\to\mathbb{R}^3,\eta_{\scriptscriptstyle E_i}:\mathbb{R}^{n_i}\to\mathbb{T}$ the position and Euler-angle orientation of agent $i$'s end-effector. More specifically, it holds that: 
\begin{align*}
p_{\scriptscriptstyle E_i}(q_i) & = p_{\scriptscriptstyle B_i} + R_{\scriptscriptstyle B_i}(\eta_{\scriptscriptstyle B_i})k_{p_i}(\alpha_i), \notag \\
\eta_{\scriptscriptstyle E_i}(q_i) & = \eta_{\scriptscriptstyle B_i} + k_{\eta_i}(\alpha_i),
\end{align*}
where $k_{p_i}:\mathbb{R}^{n_{\alpha_i}}\to\mathbb{R}^3,k_{\eta_i}:\mathbb{R}^{n_{\alpha_i}}\to\mathbb{T}$ are the forward kinematics of the robotic arm \cite{siciliano}, and $R_{\scriptscriptstyle B_i}:\mathbb{T}\to SO(3)$ is the rotation matrix of the agent $i$'s base.

Let also $v_i=[\dot{p}^\top_{\scriptscriptstyle E_i},\omega^\top_i]^\top:\mathbb{R}^{n_i}\times\mathbb{R}^{n_i}\rightarrow\mathbb{R}^{6}$ denote a function that represents the generalized velocity of agent $i$'s end-effector, with  $\omega_{i}:\mathbb{R}^{n_i}\times\mathbb{R}^{n_i}\rightarrow\mathbb{R}^{3}$ being the angular velocity. 
Then, $v_i$ can be computed as: 
\begin{align}
	v_i(q_i,\dot{q}_i) & = 
	\begin{bmatrix}
		\dot{p}_{\scriptscriptstyle E_i}(q_i)\\ 
		\omega_i(q_i,\dot{q}_i)
	\end{bmatrix} \notag \\
	& =
	\begin{bmatrix}
		\dot{p}_{\scriptscriptstyle B_i} - S(R_{\scriptscriptstyle B_i}(\eta_{\scriptscriptstyle B_i})k_{p_i}(\alpha_i))\omega_{\scriptscriptstyle B_i}(q_i,\dot{q}_i) + R_{\scriptscriptstyle B_i}(\eta_{\scriptscriptstyle B_i})\tfrac{\partial k_{p_i}(\alpha_i)}{\partial \alpha_i}\\ 
		\omega_{\scriptscriptstyle B_i}(q_i,\dot{q}_i) + R_{\scriptscriptstyle B_i}(\eta_{\scriptscriptstyle B_i})J_{\scriptscriptstyle A_i}(q_i)\dot{\alpha}_i \label{eq:diff_kinematics}
	\end{bmatrix},
\end{align}
where $J_{\scriptscriptstyle A_i}:\mathbb{R}^{n_{\alpha_i}}\to\mathbb{R}^{3\times n_{\alpha_i}}$ is the angular Jacobian of the robotic arm with respect to the agent's base \cite{siciliano}. The differential kinematics \eqref{eq:diff_kinematics} can be written as:
\begin{equation}
	v_i(q_i,\dot{q}_i) = 
	%\begin{bmatrix}
	%\dot{p}_{\scriptscriptsty\vle E_i}(q_i,\dot{q}_i) \\ \omega_{\scriptscriptstyle E_i}(q_i,\dot{q}_i)
	%\end{bmatrix} = 
	J_i(q_i)\dot{q}_i,	\label{eq:diff_kinematics_2}
\end{equation}
where the Jacobian matrix $J_i:\mathbb{R}^{n_i}\to\mathbb{R}^{6\times n_i}$ is given by: $$J_i(q_i) \coloneqq \begin{bmatrix}
I_3 & -S(R_{\scriptscriptstyle B_i}(\eta_{\scriptscriptstyle B_i})k_{p_i}(\alpha_i))J_{\scriptscriptstyle B_i}(\eta_{\scriptscriptstyle B_i}) & R_{\scriptscriptstyle B_i}(\eta_{\scriptscriptstyle B_i})\tfrac{\partial k_{p_i}(\alpha_i)}{\partial \alpha_i} \\
0_{3\times 3} & J_{\scriptscriptstyle B_i}(\eta_{\scriptscriptstyle B_i}) & R_{\scriptscriptstyle B_i}(\eta_{\scriptscriptstyle B_i})J_{\scriptscriptstyle A_i}(q_i)
\end{bmatrix}.$$  

\begin{remark}
	Note that $J_{\scriptscriptstyle B_i}(\phi_{\scriptscriptstyle B_i}, \theta_{\scriptscriptstyle B_i}, \psi_{\scriptscriptstyle B_i})$ becomes singular at representation singularities \footnote{Other representations, such as rotation matrices or unit quaternions, do not exhibit such singularities. The prevent, however, global stabilization by continuous controllers due to topological obstructions.}, when $\theta_{\scriptscriptstyle B_i} = \pm \tfrac{\pi}{2}$ and $J_i(q_i)$ becomes singular at kinematic singularities defined by the set $$\mathcal{Q}_i = \{q_i\in\mathbb{R}^{n_i} : \det(J_i(q_i)[J_i(q_i)]^\top) = 0 \}, i \in \mathcal{N}.$$ In the following, we will aim at guaranteeing that $q_i$ will always be in the closed set: $$\widetilde{\mathcal{Q}}_i = \{q_i\in\mathbb{R}^{n_i} : \lvert \det(J_i(q_i)[J_i(q_i)]^\top) \rvert \geq \varepsilon > 0\}, i \in \mathcal{N},$$ for a small positive constant $\varepsilon$. 
\end{remark}

The joint-space dynamics for agent $i\in\mathcal{N}$ can be computed using the Lagrangian formulation:
\begin{equation}
	B_{i}(q_i)\ddot{q}_i+N_{i}(q_i,\dot{q}_i)\dot{q}_i+g_{q_i}(q_i) = \tau_{i} - [J_i(q_i)]^\top\lambda_{i},   \label{eq:manipulator joint space dynamics}
\end{equation}
where $B_{i}:\mathbb{R}^{n_i}\rightarrow\mathbb{R}^{n_i\times n_i}$ is the joint-space positive definite inertia matrix,  $N_i:\mathbb{R}^{n_i}\times\mathbb{R}^{n_i}\rightarrow\mathbb{R}^{n_i\times n_i}$ represents the joint-space Coriolis matrix, $g_{q_i}:\mathbb{R}^{n_i}\rightarrow\mathbb{R}^{n_i}$
is the joint-space gravity vector, 
%$d_i:\mathbb{R}_{\geq0}\to\mathbb{R}^{n_i}$ is a bounded disturbance vector representing exogenous disturbances and modeling uncertainties, satisfying $\lVert d_{q_i}(t) \rVert\leq\bar{d}_{q_i} < \infty, \forall t\in\mathbb{R}_{\geq 0}, i\in\mathcal{N}$,
$\lambda_{i}\in\mathbb{R}^{6}$ is the generalized force vector that agent $i$ exerts on the object and $\tau_i\in\mathbb{R}^{n_i}$ is the vector of generalized joint-space inputs, with $\tau_i = [\lambda^\top_{\scriptscriptstyle B_i}, \tau^\top_{\alpha_i}]^\top$, where $\lambda_{\scriptscriptstyle B_i} = [f^\top_{\scriptscriptstyle B_i}, \mu^\top_{\scriptscriptstyle B_i}]^\top \in \mathbb{R}^6$ is the generalized force vector on the center of mass of the agent's base and $\tau_{\alpha_i}\in\mathbb{R}^{n_{\alpha_i}}$ is the torque inputs of the robotic arms' joints. By differentiating \eqref{eq:diff_kinematics_2} we obtain:
\begin{equation}
\dot{v}_i(q_i,\dot{q}_i) = J_i(q_i)\ddot{q}_i + \dot{J}_i(q_i)\dot{q}_i. \label{eq:acceleration}
\end{equation}
By inverting \eqref{eq:manipulator joint space dynamics} and using \eqref{eq:diff_kinematics_2} and \eqref{eq:acceleration}, we can obtain the task-space agent dynamics \cite{siciliano}:
\begin{align}
	& M_{i}(q_i)v^d_i(q_i,\dot{q}_i,\ddot{q}_i)+C_{i}(q_i,\dot{q}_i)v_i(q_i,\dot{q}_i)+g_{i}(q_i) =   u_{i} - \lambda_{i},   \label{eq:manipulator task space dynamics}
\end{align}
with the corresponding task-space terms $M_i:\mathbb{R}^{n_i}\backslash \mathcal{Q}_i \to\mathbb{R}^{6\times 6}$, $C_i:\mathbb{R}^{n_i}\backslash \mathcal{Q}_i \times \mathbb{R}^{n_i}\to\mathbb{R}^{6\times 6}$, $g_i:\mathbb{R}^{n_i}\backslash \mathcal{Q}_i\to\mathbb{R}^6$: 
\begin{align*}
M_i(q_i) & = \left[J_i(q_i)B^{-1}_i(q_i)J^\top_i(q_i)\right]^{-1}, \notag \\
C_i(q_i,\dot{q}_i)J_i(q_i)\dot{q}_i & = M_i(q_i)\left[J_i(q_i)B^{-1}_i(q_i)N_i - \dot{J}_i(q_i)\right]\dot{q}_i, \notag \\
g_i(q_i) & =  M_i(q_i)J_i(q_i)B^{-1}_i(q_i)g_{q_i}(q_i).
\end{align*}
The task-space input wrench $u_i$ can be translated to the joint space inputs $\tau_i\in\mathbb{R}^{{n}_i}$ via
\begin{equation}
	\tau_{i} = [J_{i}(q_i)]^{\top}u_{i}+\bar{\tau}_i(q_i), \label{eq:tau and u}
\end{equation}
where $\bar{\tau}_i$ belongs to the nullspace of $[J_{i}(q_i)]^{\top}$ and concerns over-actuated agents (see \cite{siciliano}). The term $\bar{\tau}_i$ concerns over-actuated agents and does not contribute to end-effector forces.

We define by $\mathcal{A}_i: \mathbb{R}^{n_i} \rightrightarrows \mathbb{R}^3, i \in \mathcal{N}$, the union of the ellipsoids 
that bound the $i$-th agent's volume, i.e., which is essentially the union of the ellipsoids that bound the volume of the agents' links.

\subsubsection{Object Dynamics}  \label{subsubsec:object dynamics}

Regarding the object, we denote its state as $x_{\scriptscriptstyle O}\in\mathbb{M}$, $v_{\scriptscriptstyle O}=[v^\top_{\scriptscriptstyle L,O}, \omega^\top_{\scriptscriptstyle O}]^\top \in\mathbb{R}^6$, representing the pose and velocity of the object's center of mass, with $x_{\scriptscriptstyle O} = [p^\top_{\scriptscriptstyle O}, \eta^\top_{\scriptscriptstyle O}]^\top$, $p_{\scriptscriptstyle O}\in\mathbb{R}^3$, $\eta_{\scriptscriptstyle O} = [\phi_{\scriptscriptstyle O}, \theta_{\scriptscriptstyle O}, \psi_{\scriptscriptstyle O}]^\top\in\mathbb{T}$. The second order Newton-Euler dynamics of the object are given by:
\begin{subequations} \label{eq:object dynamics} 
	\begin{align}
		\dot{x}_{\scriptscriptstyle O} & = [J_{\scriptscriptstyle O_r}(\eta_{\scriptscriptstyle O})]^{-1}v_{\scriptscriptstyle O}, \label{eq:object dynamics_1} \\ 
		\lambda_{\scriptscriptstyle O} & = M_{\scriptscriptstyle O}(x_{\scriptscriptstyle O})\dot{v}_{{\scriptscriptstyle O}}+C_{{\scriptscriptstyle O}}(x_{\scriptscriptstyle O},v_{\scriptscriptstyle O})v_{{\scriptscriptstyle O}}+g_{\scriptscriptstyle O}(x_{\scriptscriptstyle O}),\label{eq:object dynamics_2} 
	\end{align}
\end{subequations}
where $M_{\scriptscriptstyle O}:\mathbb{M}\rightarrow\mathbb{R}^{6\times6}$ is the positive definite inertia matrix, $C_{{\scriptscriptstyle O}}:\mathbb{M}\times\mathbb{R}^6\rightarrow\mathbb{R}^{6\times6}$ is the Coriolis matrix, and $g_{\scriptscriptstyle O}:\mathbb{M}\rightarrow\mathbb{R}^{6}$ is the gravity vector. In addition, $J_{\scriptscriptstyle O_r}:\mathbb{T}\rightarrow\mathbb{R}^{6\times6}$ is the object representation Jacobian $J_{\scriptscriptstyle O_r}(\eta_{\scriptscriptstyle O}) \coloneqq \text{diag}\{I_3, J_{\scriptscriptstyle O_{r,\theta}}(\eta_{\scriptscriptstyle O}) \}$, with $$J_{\scriptscriptstyle O_{r,\theta}}(\eta_{\scriptscriptstyle O}) = \begin{bmatrix}
1 & 0 & \sin(\theta_{\scriptscriptstyle O}) \\
0 & \cos(\phi_{\scriptscriptstyle O}) & -\cos(\theta_{\scriptscriptstyle O})\sin(\phi_{\scriptscriptstyle O}) \\
0 & \sin(\phi_{\scriptscriptstyle O}) & \cos(\theta_{\scriptscriptstyle O})\cos(\phi_{\scriptscriptstyle O})
\end{bmatrix},$$ which is singular when $\theta_{\scriptscriptstyle O}= \pm \tfrac{\pi}{2}$. Finally, $\lambda_{\scriptscriptstyle O}\in\mathbb{R}^6$ is the force vector acting on the object's center of mass. Also, similarly to the robotic agents, we define by $\mathcal{C}_{\scriptscriptstyle O}: \mathbb{M}\rightrightarrows \mathbb{R}^3$ the bounding ellipsoid of the object.

\subsubsection{Coupled Dynamics} \label{subsubsec: coupled dynamics}

Consider $N$ robotic agents rigidly grasping an object. Then, the coupled system object-agents behaves like a closed-chain robot and we can express the object's pose and velocity as a function of $q_i$ and $\dot{q}_i$, $\forall i\in\mathcal{N}$. 
Hence, In view of Fig. \ref{fig:Two-robotic-arms}, we conclude that:
\begin{subequations} \label{eq:coupled kinematics}
	\begin{align}
		p_{{\scriptscriptstyle O}} = p_{{\scriptscriptstyle O_i}}(q_i) &\coloneqq p_{{\scriptscriptstyle E_{i}}}(q_i) + p_{{\scriptscriptstyle O/E_i}}(q_i)  \notag\\
		& \coloneqq p_{{\scriptscriptstyle E_{i}}}(q_i) + R_{{\scriptscriptstyle E_i}}(q_i) p^{\scriptscriptstyle E_i}_{{\scriptscriptstyle O/E_i}} \label{eq:coupled kinematics_p},\\ 
		\eta_{\scriptscriptstyle O} = \eta_{\scriptscriptstyle O_i}(q_i) &=  \eta_{\scriptscriptstyle E_i}(q_i) +  \eta_{\scriptscriptstyle O/E_i}, \label{eq:coupled kinematics_ksi} 	
	\end{align}
\end{subequations}
for every $i\in\mathcal{N}$, where $p_{\scriptscriptstyle O_i}:\mathbb{R}^{n_i}\to\mathbb{R}^{3}, \eta_{\scriptscriptstyle O_i}:\mathbb{R}^{n_i}\to\mathbb{}$ are local functions of the agents that provide the object's pose, $p^{\scriptscriptstyle E_i}_{{\scriptscriptstyle O/E_i}}$ represents the constant distance and  $\eta_{\scriptscriptstyle O/E_i}$ the relative orientation offset between the $i$th agent's end-effector and the object's center of mass, which are considered known. 
%We also define the velocities $v_i, v_{\scriptscriptstyle O}:\mathbb{R}^{\geq 0}\rightarrow\mathbb{R}^6$, with 
%\begin{align}
%v_i(t) = \begin{bmatrix}
%\dot{p}_{\scriptscriptstyle E_i}(q(t),\dot{q}(t)) \\
%\omega_{\scriptscriptstyle E_i}(q(t),\dot{q}(t))
%\end{bmatrix} \label{eq:v_e_i}\\
%v_{\scriptscriptstyle O}(t) = \begin{bmatrix}
%\dot{p}_{\scriptscriptstyle O}(q(t),\dot{q}(t)) \\
%\omega_{\scriptscriptstyle O}(q(t),\dot{q}(t)).
%\end{bmatrix}. \label{eq:v_o}
%\end{align}
The grasp rigidity implies that $\omega_{\scriptscriptstyle E_i}(q_i)=\omega_{\scriptscriptstyle O}$, $\forall i\in\mathcal{N}$. Therefore, by differentiating \eqref{eq:coupled kinematics_p}, we can also express $v_{\scriptscriptstyle O}$ as a function of $q_i, \dot{q}_i$ as  
\begin{equation}
	v_{{\scriptscriptstyle O}} = v_{{\scriptscriptstyle O_i}}(q_i,\dot{q}_i) \coloneqq J_{{\scriptscriptstyle i_O}}(q_i)v_i(q_i,\dot{q}_i), \label{eq:object-end-effector jacobian}
\end{equation}
from which, we obtain:
\begin{align}
 \dot{v}_{\scriptscriptstyle O_i}(q_i,\dot{q}_i) = J_{{\scriptscriptstyle i_O}}(q_i)v^d_i(q_i,\dot{q}_i,\ddot{q}_i) + \dot{J}_{\scriptscriptstyle i_O}(q_i) v_i(q_i,\dot{q}_i),\label{eq:object-end-effector jacobian_dot}
\end{align}
where $J_{\scriptscriptstyle i_O}:\mathbb{R}^{n_i}\rightarrow\mathbb{R}^{6\times6}$ is a smooth mapping representing the Jacobian from the object to the $i$-th agent:  
\begin{equation}
	J_{{\scriptscriptstyle i_O}}(q_i)=\left[\begin{array}{cc}
		I_3 & S(p_{{\scriptscriptstyle E_i/O}}(q_i))\\
		0_{{\scriptscriptstyle 3\times3}} & I_3
	\end{array}\right],
	\label{eq:jacobian O_i}
\end{equation}
and is always full rank due to the grasp rigidity.

\begin{remark}
	Since the geometric object parameters $p^{\scriptscriptstyle E_i}_{\scriptscriptstyle O/E_i}$ and $\eta_{\scriptscriptstyle O/E_i}$ are known, each agent can compute $p_{\scriptscriptstyle O}(q_i),\eta_{\scriptscriptstyle O}(q_i)$ and $v_{\scriptscriptstyle O}(q_i,\dot{q}_i)$  by \eqref{eq:coupled kinematics} and \eqref{eq:object-end-effector jacobian}, respectively, without employing any sensory data. In the same vein, all agents can also compute the object's bounding ellipsoid $\mathcal{C}_{\scriptscriptstyle O}(x_{\scriptscriptstyle O}(q_i))$.
\end{remark}

The Kineto-statics duality \cite{siciliano} along with the grasp rigidity suggest that the 
force $\lambda_{\scriptscriptstyle O}$ acting on the object center of mass and the generalized forces $\lambda_i, i\in\mathcal{N}$, exerted by the
agents at the contact points are related through: 
\begin{equation}
	\lambda_{\scriptscriptstyle O}=[G(q)]^\top\lambda,\label{eq:grasp matrix}
\end{equation}
where $\lambda=[\lambda^{\top}_{1}, \cdots, \lambda^{\top}_{N}]^{\top}\in\mathbb{R}^{6N}$ and $G:\mathbb{R}^{{n}}\rightarrow\mathbb{R}^{6N\times6}$ is the grasp matrix, with $G(q)=[[J_{{\scriptscriptstyle O_1}}(q_1)]^{\top},\cdots,[J_{{\scriptscriptstyle O_N}}(q_N)]^{\top}]^{\top}$ and $J_{\scriptscriptstyle O_i}(q_i) \coloneqq [J_{\scriptscriptstyle i_O}(q_i)]^{-1}$ the matrix inverse of $J_{\scriptscriptstyle i_O}(q_i)$, $i\in\mathcal{N}$.

%Next, we substitute \eqref{eq:object-end-effector jacobian} and \eqref{eq:object-end-effector jacobian_dot} in \eqref{eq:manipulator task space dynamics} and we obtain in vector form after rearranging terms:
%\begin{align}
%	&\lambda = u - M(q)\dot{v} - C(q)v - g(q) - d, \label{eq:coupled dynamics 1}
%\end{align}
%where we have used the stack forms $M =  \text{diag}\{\left[M_{i}\right]_{i\in\mathcal{N}}\}$, $C = \text{diag}\{\left[C_{i}\right]_{i\in\mathcal{N}}\}$, $v=[v^\top_1,\dots,v^\top_N]^\top, g=[g_{1}^{\top}, \dots, g_{N}^{\top}]^{\top}, d=[d^\top_1,\dots,d^\top_N]^\top$, and $u=[u_{1}^{\top}, \dots, u_{N}^{\top}]^{\top}$.
Consider now the constants $c_i$, with $0< c_i < 1$ and $\sum\limits_{i\in\mathcal{N}}c_i = 1$, that play the role of load sharing coefficients for the agents. Then \eqref{eq:object dynamics_2} can be written as: 
\begin{align*}
& \sum\limits_{i\in\mathcal{N}}c_i\Big\{M_{\scriptscriptstyle O}(x_{\scriptscriptstyle O_i}(q_i))v^d_{\scriptscriptstyle O_i}(q_i,\dot{q}_i,\ddot{q}_i)  + g_{\scriptscriptstyle O}(x_{\scriptscriptstyle O_i}(q_i))  C_{\scriptscriptstyle O}(x_{\scriptscriptstyle O_i}(q_i), v_{\scriptscriptstyle O_i}(q_i,\dot{q}_i))v_{\scriptscriptstyle O_i}(q_i,\dot{q}_i) \Big\} \notag \\ 
&\hspace{120mm} = \sum\limits_{i\in\mathcal{N}}[J_{\scriptscriptstyle O_i}(q_i)]^\top\lambda_i,
\end{align*}
from which, by employing \eqref{eq:grasp matrix}, \eqref{eq:diff_kinematics_2}, \eqref{eq:acceleration}, \eqref{eq:object-end-effector jacobian} and \eqref{eq:object-end-effector jacobian_dot}, and after straightforward algebraic manipulations, we obtain the coupled dynamics:
\begin{align}
	&\sum\limits_{i\in\mathcal{N}}\Big\{\widetilde{M}_i(q_i)\ddot{q}_i + \widetilde{C}_i(q_i,\dot{q}_i)\dot{q}_i + \widetilde{g}_i(q_i) \Big\} = \sum\limits_{i\in\mathcal{N}}[J_{\scriptscriptstyle O_i}(q_i)]^\top u_i,
	\label{eq:coupled dynamics 2}
\end{align}	
where:
\begin{align*}
	\widetilde{M}_i(q_i) & \coloneqq c_iM_{\scriptscriptstyle O}(x_{\scriptscriptstyle O_i}(q_i))J_{i_{\scriptscriptstyle O}}(q_i)J_i(q_i) + [J_{\scriptscriptstyle O_i}(q_i)]^\top M_i(q_i)J_i(q_i) ,\\
	\widetilde{C}_i(q_i,\dot{q}_i) & \coloneqq [J_{\scriptscriptstyle O_i}(q_i)]^\top \Big( M_i(q_i)\dot{J}_i(q_i) + C_i(q_i,\dot{q}_i)J_i(q_i)\Big) + c_iM_{\scriptscriptstyle O}(x_{\scriptscriptstyle O_i}(q_i))J_{i_{\scriptscriptstyle O}}(q_i)\dot{J}_i(q_i), \notag \\
	&\hspace{47mm} + c_iM_{\scriptscriptstyle O}(x_{\scriptscriptstyle O_i}(q_i))\dot{J}_{i_{\scriptscriptstyle O}}(q_i)J_i(q_i) + c_iC_{\scriptscriptstyle O}(x_{\scriptscriptstyle O_i}(q_i),v_{\scriptscriptstyle O_i}(q_i,\dot{q}_i)), \\
\widetilde{g}_i(q_i) & \coloneqq c_ig_{\scriptscriptstyle O}(x_{\scriptscriptstyle O_i}(q_i)) + [J_{\scriptscriptstyle O_i}(q_i)]^\top g_i(q_i),
\end{align*}
where $x_{\scriptscriptstyle O_i} \coloneqq [\dot{p}^\top_{\scriptscriptstyle O_i}, \eta^\top_{\scriptscriptstyle O_i}]^\top\in \mathbb{M}$ and $i\in\mathcal{N}$.
%From \eqref{eq:J_oi bound} and \eqref{eq:d_i bound}, we obtain that 
%\begin{equation}
%\lVert \widetilde{d}_i(q,t) \rVert \leq \hat{d}_i \coloneqq (\lVert p^{\scriptscriptstyle E_i}_{\scriptscriptstyle O/E_i} + 1\rVert)(c_i\bar{d}_{\scriptscriptstyle O} + \bar{d}_i), \label{eq:d_tilde bound}
%\end{equation}
%$\forall q\in\mathbb{R}^{n},t\in\mathbb{R}_{\geq 0}, i\in\mathcal{N}$.
%\end{subequations}
%Notice from \eqref{eq:coupled terms_M} that $\widetilde{M}$ is a positive definite matrix, owing to the positive definiteness of $M_{\scriptscriptstyle O}$ and $M_i, \forall i\in\mathcal{N}$ and the full column rank of $G$. 

\begin{remark}
	Note that the agents dynamics under consideration hold for generic robotic agents comprising of a moving base and a robotic arm. Hence, the considered framework can be applied for mobile, aerial, or underwater manipulators. 
\end{remark}

\begin{problem}
	Consider $N$ robotic agents, rigidly grasping an object, governed by the coupled dynamics \eqref{eq:coupled dynamics 2}. Given a desired pose $x_{\text{des}}$ for the object, design the control inputs $u_i\in\mathbb{R}^{6N}$ such that $\lim\limits_{t\to\infty} \| x_{\scriptscriptstyle O}(t) - x_{\text{des}} \| \to 0$, while ensuring the satisfaction of the following collision avoidance and singularity properties:
	\begin{enumerate}
		\item $\mathcal{A}_i(q_i(t))\cap\mathcal{O}_z = \emptyset, \forall i\in\mathcal{N}, z\in\mathcal{Z}$,
		\item $\mathcal{C}_{\scriptscriptstyle O}(x_{\scriptscriptstyle O}(t))\cap\mathcal{O}_z = \emptyset, \forall z \in \mathcal{Z}$,
		\item $\mathcal{A}_i(q_i(t))\cap\mathcal{A}_{j}(q_{j}(t)) = \emptyset, \forall i, j \in \mathcal{N}, i\neq j$,
		\item $-\tfrac{\pi}{2} < -\bar{\theta}\leq \theta_{\scriptscriptstyle O}(t) \leq \bar{\theta} < \tfrac{\pi}{2}$,
		\item $-\tfrac{\pi}{2} < -\bar{\theta}\leq \theta_{\scriptscriptstyle B_i}(t) \leq \bar{\theta} < \tfrac{\pi}{2}$, $\forall i \in \mathcal{N}$,
		\item $q_i \in \widetilde{\mathcal{Q}}_i$, $\forall i \in \mathcal{N}$,
	\end{enumerate} 
	$\forall t\in\mathbb{R}_{\geq 0}$, for $0 < \bar{\theta} < \tfrac{\pi}{2}$,  as well as the velocity and input constraints: $\lvert \tau_{i_k}(t) \rvert \leq \bar{\tau}_i, \lvert \dot{\tau}_{i_k}(t) \rvert \leq \bar{\dot{\tau}}_i, \lvert \dot{q}_{i_k}(t) \rvert \leq \bar{\dot{q}}_i, \forall k\in\{1,\dots,n_i\}, i\in\mathcal{N}$, for some positive constants $\bar{\tau}_i, \bar{\dot{q}}_i, i\in\mathcal{N}$.
\end{problem}
Regarding the joint velocity constraints, we impose for the arm joint velocities the constraint $\rVert \dot{\alpha}_{i}(t) \lVert \leq 1, \forall i\in\mathcal{N}$, which maximizes the manipulability ellipsoid of the arms \cite{siciliano}, and hence increases robotic manipulability.
Specifications $1)-3)$ in the aforementioned problem stand for collision avoidance between the agents, the objects, and the workspace obstacles and specifications $4)-6)$ stand for representation and kinematic singularities. 
%\begin{itemize}
%	\item $1)$ stands for collision avoidance between the agents and the obstacles.
%	\item $2)$ stands for collision avoidance between the object and the obstacles.
%	\item $3)$ stands for collision avoidance between the agents.
%	\item $4)$ stands for representation singularity avoidance of the object.
%	\item $5)$ stands for representation singularity avoidance of the agents' bases.
%	\item $6)$ stands for kinematic singularity avoidance of the agents. 
%\end{itemize}

In order to solve the aforementioned problem, we need the following assumption regarding the workspace:

\begin{assumption} \label{ass:feasility_assumption}
	(Problem Feasibility Assumption) The set $\{q \in \mathbb{R}^n : \mathcal{A}_i(q_i)\cap\mathcal{O}_z = \emptyset, \mathcal{A}_i(q_i)\cap\mathcal{A}_\ell(q_\ell) = \emptyset, \mathcal{C}_i(x_{\scriptscriptstyle O_i}(q_i))\cap\mathcal{O}_z = \emptyset, \forall i,\ell\in\mathcal{N}, i\neq \ell, z\in\mathcal{Z}\}$, is connected.	
\end{assumption}

\begin{assumption} \label{ass:sensing_assumption}
(Sensing and communication capabilities) Each agent $i\in\mathcal{N}$ is able to continuously measure the other agents' state $q_j$, $j\in\mathcal{N}\backslash\{i\}$. Moreover, each agent $i\in\mathcal{N}$ is able to communicate with the other agents $j\in\mathcal{N}\backslash\{i\}$ without any delays.
\end{assumption}
Note that the aforementioned sensing assumption is reasonable, since in cooperative manipulation tasks, the agents are sufficiently close to each other, and therefore potential sensing radii formed by realistic sensors are large enough to cover them. Moreover, each agent $i\in\mathcal{N}$ can construct at every time instant the set-valued functions $\mathcal{A}_j(q_j)$, $\forall j\in\mathcal{N}\backslash\{i\}$, whose structure can be transmitted off-line to all agents. Let us define also the sets: 
\begin{align*}
\mathcal{S}_{i, \scr O} & \coloneqq \{q_i\in\mathbb{R}^{n_i} : \mathcal{A}_i(q_i)\cap\mathcal{O}_z \neq \emptyset, \forall z \in \mathcal{Z} \}, \notag \\
\mathcal{S}_{i, \scriptscriptstyle A} & \coloneqq \{q\in\mathbb{R}^n : \mathcal{A}_i(q_i)\cap\mathcal{A}_{j}(q_\ell) \neq \emptyset, \forall \ell \in \mathcal{N} \backslash \{i\} \},
\end{align*}
\begin{align*}
\widetilde{\mathcal{S}}_{i, \scriptscriptstyle A}([q_\ell]_{\ell\in\mathcal{N}\backslash\{i\}}) & \coloneqq \{q_i\in\mathbb{R}^{n_i} : q\in \mathcal{S}_{i,\scr A} \} \notag \\
\mathcal{S}_{\scriptscriptstyle O_i} & \coloneqq \{ q_i\in\mathbb{R}^{n_i}: \mathcal{C}_{\scriptscriptstyle O}(x_{\scriptscriptstyle O_i}(q_i))\cap\mathcal{O}_z \neq \emptyset, \forall z \in \mathcal{Z}\}, 
\end{align*}
for every $i \in \mathcal{N}$, associated with the desired collision-avoidance properties, where the notation $[q_\ell]_{\ell\in\mathcal{N}\backslash\{i\}}$ stands for the stack vector of all $q_\ell$, $\ell\in\mathcal{N}\backslash\{i\}$. Moreover, define the projection sets for agent $i$ as the set-valued functions: 
\begin{align*}
\widetilde{\mathcal{S}}_{i,\scr A}([q_\ell]_{\scr \ell\in\mathcal{N}\backslash\{i\}}) \coloneqq \{ q_i\in\mathbb{R}^{n_i} : q\in \mathcal{S}_{i, \scr A} \}, i \in \mathcal{N},
\end{align*}
where the notation $[q_\ell]_{\scr \ell\in\mathcal{N}\backslash\{i\}}$ stands for the stack vector of all $q_\ell, \ell\in\mathcal{N}\backslash\{i\}$.

\section{Main Results} \label{sec:solution}

In this section, a systematic solution to Problem 1 is introduced. Our overall approach builds on designing a NMPC scheme for the system of the manipulators and the object. The proposed methodology is decentralized, since we do not consider a centralized system that calculates all the control signals and transmits them to the agents, like in our previous work \cite{alex_chris_med_2017}. As expected, this relaxes greatly the computational burden of the NMPC approach, which is also verified by the simulation results. 
 To achieve that, we employ a leader-follower perspective. More specifically, as will be explained in the sequel, at each sampling time, a leader agent solves part of the coupled dynamics \eqref{eq:coupled dynamics 2} via an NMPC scheme, and transmits its predicted variables to the rest of the agents.  Assume, without loss of generality, that the leader corresponds to agent $i=1$. Loosely speaking, the proposed solution proceeds as follows: agent $1$ solves, at each sampling time step, the receding horizon model predictive control subject to the forward dynamics:
\begin{equation}
	\widetilde{M}_1(q_1)\ddot{q}_1 + \widetilde{C}_1(q_1,\dot{q}_1)\dot{q}_1 + \widetilde{g}(q_1)  = [J_{\scriptscriptstyle O_1}(q_1)]^\top u_1, \label{eq:dynamics_leader}
\end{equation}
and a number of inequality constraints, as will be clarified later. After obtaining a control input sequence and a set of predicted variables for $q_1,\dot{q}_1$, denoted as $\hat{q}_1,\hat{\dot{q}}_1$, it transmits the corresponding predicted state for the object $x_{\scriptscriptstyle O_1}(\hat{q}_1), v_{\scriptscriptstyle O_1}(\hat{q}_1,\hat{\dot{q}}_1)$ for the control horizon to the other agents $\{2,\dots,N\}$. Then, the followers solve the receding horizon NMPC subject to the forward dynamics:
\begin{equation}
	\widetilde{M}_i(q_i)\ddot{q}_i + \widetilde{C}_i(q_i,\dot{q}_i)\dot{q}_i + \widetilde{g}(q_i)  = [J_{\scriptscriptstyle O_i}(q_i)]^\top u_i, \label{eq:dynamics_followers} 
\end{equation}
and the state equality constraints: 
\begin{align} \label{eq:equality state constr 1}
	& x_{\scriptscriptstyle O_i}(q_i) = x_{\scriptscriptstyle O_1}(\hat{q}_1), v_{\scriptscriptstyle O_i}(q_i,\dot{q}_i) = v_{\scriptscriptstyle O_1}(\hat{q}_1,\hat{\dot{q}}_1),
\end{align}
$i \in \{2,\dots,N\}$ as well as a number of inequality constraints that incorporate obstacle and inter-agent collision avoidance. More specifically, we consider that there is a priority sequence among the agents, which we assume, without loss of generality, that is defined by $\{1,\dots,N\}$, and can be transmitted off-line to the agents. Each agent, after solving its optimization problem, transmits its calculated predicted variables to the agents of lower priority, which take them into account for collision avoidance. Note that the coupled object-agent dynamics are implicitly taken into account in equations \eqref{eq:dynamics_leader}, \eqref{eq:dynamics_followers} in the following sense. Although the coupled model \eqref{eq:coupled dynamics 2} does not imply that each one of these equations is satisfied, by forcing each agent to comply with the specific dynamics through the optimization procedure, we guarantee that \eqref{eq:coupled dynamics 2} is satisfied, since it's the result of the addition of \eqref{eq:dynamics_leader} and \eqref{eq:dynamics_followers}, for every $i =1$ and $i \in \{2,\dots,N\}$, respectively. Intuitively, the leader agent is the one that determines the path that the object will navigate through, and the rest of the agents are the followers that contribute to the transportation.  Moreover, the equality constraints \eqref{eq:equality state constr 1} guarantee that the predicted variables of the agents $\{2,\dots,N\}$ will comply with the rigidity at the grasping points through the equality constraints \eqref{eq:equality state constr 1}.

By using the notation $x_i \coloneqq [x^\top_{i_1},x^\top_{i_2}]^\top\coloneqq[q^\top_i, \dot{q}^\top_i]^\top\in\mathbb{R}^{2n_i}$, $i\in\mathcal{N}$, the nonlinear dynamics of each agent can be written as: 
\begin{equation} \label{eq:main_system}
	\dot{x}_i = \widetilde{f}_i(x_i,u_i) \coloneqq 	
	\begin{bmatrix}
		\widetilde{f}_{i_1}(x_i) \\
		\widetilde{f}_{i_2}(x_i,u_i) 		
	\end{bmatrix}, 
\end{equation}
where $\widetilde{f}_i:E_i\times\mathbb{R}^{6}\to\mathbb{R}^{2n_i}$ is the locally Lipschitz function: 
\begin{align*}
\widetilde{f}_{i_1}(x_i,u_i) & = x_{i_2}, \\
\widetilde{f}_{i_2}(x_i,u_i) & = \widehat{M}_i(q_i)\Big( [J_{\scriptscriptstyle O_i}(q_i)]^\top u_i - \widetilde{C}_i(q_i,\dot{q}_i)\dot{q} - \widetilde{g}_i(q_i)  \Big), i \in \mathcal{N},
\end{align*}
where $\widehat{M}_i: \mathbb{R}_{n_i}\backslash\mathcal{Q}_i\to\mathbb{R}^{n_i\times6}$, is the pseudo-inverse $$\widehat{M}_i(q_i) \coloneqq \widetilde{M}_i(q_i)\Big(\widetilde{M}_i(q_i)[\widetilde{M}_i(q_i)]^\top\Big)^{-1},$$ and $E_i\coloneqq \mathbb{R}^{n_i}\backslash \mathcal{Q}_i\times\mathbb{R}^{n_i}$, $\forall i\in\mathcal{N}$. It can be proved that in the set $\mathbb{R}^{n_i}\backslash\mathcal{Q}_i$ the matrix $\widetilde{M}_i(q_i)[\widetilde{M}_i(q_i)]^\top$ has full rank and hence, $\widehat{M}_i(q_i)$ is well defined for all $q\in \mathbb{R}^{n_i}\backslash\mathcal{Q}_i$. We define then the error vector $e_1:E_1 \to\mathbb{M}\times\mathbb{R}^6$, as:
	\begin{align*}
		e_1(x_1) \coloneqq \begin{bmatrix}
			x_{\scriptscriptstyle O_1}(q_1) - x_\text{des} \\
			v_{\scriptscriptstyle O_1}(q_1,\dot{q}_1)
		\end{bmatrix},
	\end{align*}
	which gives us the \emph{error dynamics}:
	\begin{equation} \label{eq:error_dynamics}
		\dot{e}_1 = g_1(x_1,u_1),
	\end{equation}
	where the function $g_1:E_1\times\mathbb{R}^{6}\to\mathbb{R}^{2n_i}$ is given by:
\begin{align*}
	&g_1(x_1,u_1) \coloneqq 
	\begin{bmatrix}	
			[J_{\scr O_r}(\eta_{\scr O_1}(q_1))]^{-1} J_{1_{\scr O}}(q_1)J_1(q_1)\dot{q}_1 \\
			J_{1_{\scr O}}(q_1) J_1(q_1) f_{1_2}(x_1,u_1) + \Big(J_{1_{\scr O}} \dot{J}_1(q_1) + \dot{J}_{1_{\scr O}}(q_1)J_1(q_1)\Big)\dot{q}_1.		
		\end{bmatrix},
\end{align*}
where we employed \eqref{eq:error_dynamics}, \eqref{eq:object dynamics_1}, \eqref{eq:acceleration}, and \eqref{eq:object-end-effector jacobian_dot}.
\begin{remark}
It can be concluded that $g_1(\cdot,u_1)$ is \emph{Lipschitz continuous} in $E_1$ since it is continuously differentiable in its domain. Thus, for every $x_1,x_1' \in E_1$, with $x_1\neq x_1'$, there exists a Lipschitz constant $L_g$ such that: $|g(x_1,u)-g(x_1',u)| \le L_g \|x_1-x_1'\|$.
\end{remark}
From \eqref{eq:tau and u}, we have that $\dot{\tau}_i = [\dot{J}_i(q_i)]^\top u_i + [J_i(q_i)]^\top\dot{u}_i$. Hence, the constraints for $\tau_{i_k}$ and $\dot{\tau}_{i_k}$, $k\in\mathbb{R}^{n_i}$,$i\in\mathcal{N}$, can be written as coupled state-input constraints: 
\begin{align*}
\lVert \tau_i \rVert \leq \bar{\tau}_i & \Leftrightarrow \lVert [J(q_i)]^\top u_i \rVert \leq \bar{\tau}_i, \notag \\
\lVert \dot{\tau}_i \rVert \leq \bar{\dot{\tau}}_i & \Leftrightarrow \lVert [\dot{J}_i(q_i)]^\top u_i+ [J_i(q_i)]^\top\dot{u}_i \rVert \leq \bar{\dot{\tau}}_i.
\end{align*}
Let us now define the following sets $U_i \subseteq \mathbb{R}^{6\times6\times(2n_i)}$:
\begin{align}
	& U_i \coloneqq \left\{ (u_i,\dot{u}_i,x_i)\in\mathbb{R}^{6\times 6\times(2n_i)}: \lVert [J(q_i)]^\top u_i \rVert \leq \bar{\tau}_i, \right. \notag \\
	&\hspace{56mm} \left. \lVert [\dot{J}_i(q_i)]^\top u_i+ [J_i(q_i)]^\top\dot{u}_i \rVert \leq \bar{\dot{\tau}}_i \right\}, i \in \mathcal{N},
\end{align}
%\begin{align*}
%U &= \{u \in \mathbb{R}^{6N} : -\bar{\tau}_i \le \tau_{i_k} \le \bar{\tau}_i, -\bar{\dot{\tau}}_i \le \dot{\tau}_{i_k} \le \bar{\dot{\tau}}_i, \notag \\
% &-\bar{q}_i^\star \le q_{i_k} \le \bar{q}_i^\star, i \in \mathcal{N}, k \in \{1, \dots, n_i\} \},
%\end{align*}
as the sets that capture the control input constraints of \eqref{eq:main_system}, as well as their projections
\begin{align}
	U_{i,u} \coloneqq \left\{u_i\in \mathbb{R}^6: (u_i,\dot{u}_i,x_i)\in U_i \right\}, i\in\mathcal{N}.
\end{align}
Define also the set-valued functions $X_i:\mathbb{R}^{n-n_i}\rightrightarrows \mathbb{R}^{2n_i}$, $i\in\mathcal{N}$, by: 
\begin{align}
X_1([q_\ell]_{\scr \ell\in\{2,\dots,N\}}) & \coloneqq \Big\{ x_1 \in \mathbb{R}^{2n_1} : \theta_{\scriptscriptstyle O_1}(q_1)\in [-\bar{\theta}, \bar{\theta}], \theta_{\scriptscriptstyle B_1}\in [-\bar{\theta}, \bar{\theta}], \lvert \dot{q}_{k_1}\rvert \leq \bar{\dot{q}}_1, \notag \\ 
&\hspace{10mm} q_1 \in \widetilde{\mathcal{Q}}_1 \backslash \left(\mathcal{S}_{1, {\scriptscriptstyle O}} \cup \widetilde{\mathcal{S}}_{1, \scriptscriptstyle A}([q_\ell]_{\scr \ell\in\{2,\dots,N\}}) \right), x_{\scriptscriptstyle O_1}(q_1)\in\mathbb{R}^3\backslash S_{\scriptscriptstyle O_1} \Big\}, \notag \\
X_i([q_\ell]_{\scr \ell\in\mathcal{N}\backslash\{i\}}) & \coloneqq \Big\{x_i \in \mathbb{R}^{2n_i} : \theta_{\scriptscriptstyle B_i}\in [-\bar{\theta},\bar{\theta}], \lvert \dot{q}_{k_i} \rvert \leq \bar{\dot{q}}_i, \notag  q_i \in\widetilde{\mathcal{Q}}_i\backslash\left(\mathcal{S}_{i, {\scriptscriptstyle O}} \cup \mathcal{S}_{i, \scriptscriptstyle A}([q_\ell]_{\scr \ell\in\mathcal{N}\backslash\{i\}}) \right) \Big\},
\end{align}
$i\in\{2,\dots,N\}$. Note that $q_i\in X_i([q_\ell]_{\scr \ell\in\mathcal{N}\backslash\{i\}}) \implies q_i\notin\mathcal{Q}_i$, $\forall i\in\mathcal{N}$.

The sets $X_i$ capture all the state constraints of the system dynamics \eqref{eq:main_system}, i.e., representation- and singularity-avoidance, collision avoidance among the agents and the obstacles, as well as collision avoidance of the object with the obstacles, which is assigned to the leader agent only We further define the set-valued functions $\mathcal{E}_1: \mathbb{R}^{n-n_1}\rightrightarrows \mathbb{M}\times\mathbb{R}^6$ by: $$\mathcal{E}_1([q_\ell]_{\scr \ell\in\{2,\dots,N\}}) \coloneqq \{ e_1(x_1) \in \mathbb{M}\times\mathbb{R}^6 : x_1\in X_1([q_\ell]_{\scr \ell\in\{2,\dots,N\}})\}.$$ Note that only the leader agent is responsible for deriving a collision- and representation singularity-free path for the object.  

The main problem at hand is the design of a \emph{feedback control law} $u_1 \in U_1$ for agent $1$ which guarantees that the error signal $e_1$ with dynamics given in \eqref{eq:error_dynamics}, satisfies $\displaystyle \lim_{t \to \infty} \|e_1(x_1(t))\| \to 0$, while ensuring singularity avoidance, collision avoidance between the agents, between the agents and the obstacles as well as the object and the obstacles. The role of the followers $\{2,\dots,N\}$ is, through the load-sharing coefficients $c_2,\dots,c_N$ in \eqref{eq:coupled dynamics 2}, to contribute to the object trajectory execution, as derived by the leader agent $1$. In order to solve the aforementioned problem, we propose a NMPC scheme, that is presented hereafter.

Consider a sequence of sampling times $\{t_j\}$, $j \in \mathbb{N}$ with a constant sampling period $h$, $0 < h < T_p$, where $T_p$ is the prediction horizon, such that: $t_{j+1} = t_j + h$, $j \in \mathbb{N}$. Hereafter we will denote by $j$ the sampling instant. In sampled-data NMPC, a FHOCP is solved at the discrete sampling time instants $t_j$ based on the current state error information $e_1(x_1(t_j))$. The solution is an optimal control signal $\hat{u}_1^\star(s)$, computed over $s \in [t_j,t_j+T_p]$. For agent $1$, the open-loop input signal applied in between the sampling instants is given by the solution of the following FHOCP:
\begin{subequations}
	\begin{align}
	&\hspace{0mm}\min\limits_{\hat{u}_1(\cdot)} J_1(e_1(x_1(t_j)),\hat{u}_1(\cdot)) = \min\limits_{\hat{u}_1(\cdot)} \Bigg\{  V_1(e_1(\hat{x}_1(t_j+T_p))) \notag \\
	&\hspace{70mm} + \int_{t_j}^{t_j+T_p} \Big[ F_1(e_1(\hat{x}_1(s)), \hat{u}_1(s)) \Big] ds \Bigg\}  \label{eq:mpc_minimazation}
	\end{align}
	\begin{align}
	&\hspace{0mm}\text{subject to:} \notag \\
	&\hspace{1mm} \dot{e}(\hat{x}_1(s)) = g_1(\hat{x}_1(s),\hat{u}_1(s)), \ e_1(\hat{x}_1(t_j)) = e_1(x_1(t_j)),   \label{eq:diff_mpc} \\
	&\hspace{1mm} e_1(\hat{x}_1(s)) \in \mathcal{E}_{1}([q_\ell(t_j)]_{\scr \ell\in\{2,\dots,N\}}), s \in [t_j,t_j+T_p], \label{eq:mpc_constrained_set_1} \\
	&\hspace{1mm} (\hat{u}_1(s),\hat{\dot{u}}_1(s),\hat{x}_1(s))\in U_1, s \in [t_j,t_j+T_p], \label{eq:mpc_constrained_set_2} \\
	&\hspace{1mm} e_1(\hat{x}_1(t_j+T_p)) \in \mathcal{F}_1([q_\ell]_{\ell\in\{2,\dots,N\}}). \label{eq:mpc_terminal_set}
	\end{align}
\end{subequations}
At a generic time $t_j$ then, agent $1$ solves the aforementioned FHOCP. The notation $\hat{(\cdot)}$ is used to distinguish the predicted variables which are internal to the controller, corresponding to the system \eqref{eq:diff_mpc}. This means that $e_1(\hat{x}_1(\cdot))$ is the solution of \eqref{eq:diff_mpc} driven by the control input $\hat{u}_1(\cdot): [t_j, t_j+T_p] \to U_1$ with initial condition $e_1(x_1(t_j))$. Note that, since the prediction horizon is finite, the predicted values are not the same with the actual closed-loop values (see \cite{frank_2003_nmpc_bible}). In the following, we use the notation $\mathcal{E}_1(\cdot)$ instead of $\mathcal{E}_1([q_\ell]_{\ell\in\{2,\dots,N\}})$  for brevity. 
The functions $F_1 : \mathcal{E}_{1}(\cdot) \times U_{1,u} \to \mathbb{R}_{\geq 0}$, $V_1: \mathcal{E}_1(\cdot) \to \mathbb{R}_{\geq 0}$ stand for the \emph{running cost} and the \emph{terminal penalty cost}, respectively, and they are defined as: $F_1 \big(e_1, u_1\big) = e_1^{\top} Q_1 e_1 + u_1^{\top} R_1 u_1$, $V_1 \big(e_1\big) = e_1^{\top} P_1 e_1$; $R_1 \in \mathbb{R}^{6 \times 6}$ and $P_1 \in \mathbb{R}^{(2n_1) \times (2n_1)}$ are symmetric and positive definite controller gain matrices to be appropriately tuned; $Q_1 \in \mathbb{R}^{(2n_1) \times (2n_1)}$ is a symmetric and positive semi-definite controller gain matrix to be appropriately tuned.
%The \textit{terminal} set $\mathcal{F}_1$ will be defined later. For the running and  terminal penalty costs $F_1$ and $V_1$, respectively, the following hold:
%\begin{lemma} \label{lemma:F_i_bounded_K_class}
%There exist functions $\red{\xi}_1$, $\red{\xi}_2 \in \mathcal{K}_{\infty}$ such that: $\red{\xi}_1\big(\|z\|\big) \leq F_1\big(e_1, u_1\big) \leq \red{\xi}_2\big(\|z \|\big)$, for every $z \coloneqq \left[ e_1^\top, u_1^\top\right]^\top \in \mathcal{E}_{1}(\cdot) \times U_{1,u}$.
%\end{lemma}
%\begin{proof}
%	The proof can be found in Appendix \ref{app:proof_lemma_1}.
%\end{proof}
%\begin{lemma}\label{lemma:F_Lipschitz}
%The running cost $F_1$ is Lipschitz continuous in $\mathcal{E}_{1}(\cdot) \times U_{1,u}$. Thus, it holds that: $$\big|F_1(e_1, u_1) - F_1(e'_1, u_1)\big| \leq L_{\scriptscriptstyle F_1} \|e_1 - e_1'\|, \forall e_1, e_1' \in \mathcal{E}_1(\cdot), u_1 \in U_{1,u},$$ where $L_{\scriptscriptstyle F_1} \coloneqq 2 \sigma_{\max}(Q_1) \sup\limits_{e_1 \in \mathcal{E}_{1}} \|e_1\|$.
%\end{lemma}
%\begin{proof}
%	The proof can be found in Appendix \ref{app:proof_of_F_lipsitz}.
%\end{proof}
\noindent The \emph{terminal set} $\mathcal{F}_1 \subseteq \mathcal{E}_1(\cdot)$ is chosen as: $$\mathcal{F}_1([q_\ell]_{\ell\in\{2,\dots,N\}}) = \{e_1 \in \mathcal{E}_1([q_\ell]_{\ell\in\{2,\dots,N\}}): V_1(e_1) \le \epsilon_1 \},$$ where $\epsilon_1 \in \mathbb{R}_{> 0}$ is an arbitrarily small constant to be appropriately tuned. The terminal set is used to enforce the stability of the closed-loop system. 

The solution to FHOCP \eqref{eq:mpc_minimazation} - \eqref{eq:mpc_terminal_set} at time $t_j$ provides an optimal control input, denoted by
$\hat{u}_1^{\star}(s;\ e_1(x_1(t_j)),x_1(t_j))$, $s \in [t_j, t_j + T_p]$. This control input is then applied to the system until the next sampling instant $t_{j+1}$:
\begin{align}
u_1\left(s; \ x_1(t_j), e_1(x_1(t_j))\right) = \hat{u}_1^{\star}\left(s; \ x_1(t_j), e_1(x_1(t_j))\right),
\label{eq:optimal_input}
\end{align}
for every $s \in [t_j, t_{j}+h)$. At time $t_{j+1} = t_{j}+h$ a new FHOCP is solved in the same manner, leading to a receding horizon approach. The control input $u_1(\cdot)$ is of feedback form, since it is recalculated at each sampling instant based on the then-current state. The solution of \eqref{eq:error_dynamics} at time $s$, $s \in [t_j, t_j+T_p]$, starting at time $t_j$, from an initial condition $x_1(t_j), e_1(x_1(t_j))$, by application of the control input $u_1 : [t_j, s] \to U_{1,u}$ is denoted by: $$e_1\big(x_1(s); \ u_1(\cdot); \ x_1(t_j), e_1(x_1(t_j)
)\big), s \in [t_j, t_j+T_p].$$ The \textit{predicted} state of the system \eqref{eq:diff_mpc} at time $s$, $s \in [t_j, t_j+T_p]$ based on the measurement of the state at time $t_j$, $x_1(t_j)$, by application of the control input $u_1\big(t;\ x_1(t_j), e_1(x_1(t_j))\big)$ as in \eqref{eq:optimal_input}, is denoted by $\hat{x}_1\big(s;\ u_1(\cdot); \ x_1(t_j), e_1(x_1(t_j))\big)$, and the corresponding predicted error by: $$e_1(\hat{x}_1(\cdot); \ u_1(\cdot); \ x_1(t_j), e_1(x_1(t_j) )\big), s \in [t_j, t_j+T_p].$$

\noindent After the solution of the FHOCP and the calculation of the predicted states $$\hat{x}_1\big(s;\ u_1(\cdot), e_1(x_1(t_j)),x_1(t_j)\big), s \in [t_j, t_j+T_p],$$ at each time instant $t_j$, agent $1$ transmits the values $\hat{q}_1(s,\cdot)$, $\hat{\dot{q}}_1(s,\cdot)$ as well as $x_{\scr O_1}(\hat{q}_1(s,\cdot))$ and $v_{\scr O_1}(\hat{q}_1(s,\cdot),\hat{\dot{q}}_1(s,\cdot))$, as computed by \eqref{eq:coupled kinematics}, \eqref{eq:object-end-effector jacobian}, $\forall s \in [t_j, t_j+T_p]$ to the rest of the agents $\{2,\dots,N\}$. The rest of the agents then proceed as follows. Each agent $i\in\{2,\dots,N\}$, solves the following FHOCP:
\begin{subequations} \label{eq:mpc_followers}
		\begin{align}
		&\hspace{0mm}\min\limits_{\hat{u}_i(\cdot)} J_i(x_i(t_j)),\hat{u}_i(\cdot)) \label{eq:mpc_minimazation followers} \\
		&\hspace{0mm}\text{subject to:} \notag \\
		&\hspace{1mm} \dot{\hat{x}}_i = \widetilde{f}_i(x_i(s),u_i(s)),  \label{eq:diff_mpc followers} \\
		&\hspace{1mm} \hat{x}_i(s) \in X_i\Big([\hat{q}_\ell(t_j)]_{\ell\in\{i+1,\dots,N\}}\Big), \label{eq:mpc_constrained_set_1 followers} \\
		&\hspace{1mm} \hat{x}_i(s) \in X_i\Big([\hat{q}_\ell(s,\cdot)]_{j\in\{1,\dots,i-1\}}\Big), \label{eq:mpc_constrained_set_1 followers all s} \\
		&\hspace{1mm} \hat{x}_{\scr O_i}(\hat{q}_i(s)) = x_{\scr O_1}(\hat{q}_1(s;\cdot)), \label{eq:equality constr_1 followers} \\
		&\hspace{1mm} \hat{v}_{\scr O_i}(\hat{q}_i(s),\dot{q}_i(s)) = v_{\scr O_1}(\hat{q}_1(s;\cdot), \dot{\hat{q}}_1(s;\cdot)), \label{eq:equality constr_2 followers} \\
		&\hspace{1mm} (u_i(s),\dot{u}_i(s),x_i(s))\in U_i, s \in [t_j,t_j+T_p], \label{eq:mpc_constrained_set_2 followers} 
		\end{align}
	\end{subequations}
at every sampling time $t_j$. Note that, through the equality constraints \eqref{eq:equality constr_1 followers}, \eqref{eq:equality constr_2 followers}, the follower agents must comply with the trajectory  computed by the leader $\hat{q}_1(s,\cdot), \hat{\dot{q}}_1(s,\cdot)$. This can be problematic in the sense that this trajectory might drive the followers to collide with an obstacle or among each other, i.e., a solution to \eqref{eq:mpc_followers} might not exist. Resolution of such cases is not in the scope of this paper and constitutes part of future research. Here, we assume that there are no such cases: 
\begin{assumption} \label{ass:follower feasibility}
	The sets 
	$\{ (q,s)\in\mathbb{R}^n\times[t_j, t_j+T_p] :x_{\scr O_i}(q_i(s)) = x_{\scr O_1}(\hat{q}_1(s;\cdot)),v_{\scr O_i}(q_i(s),\dot{q}_i(s)) = v_{\scr O_1}(\hat{q}_1(s;\cdot), \hat{\dot{q}}_1(s;\cdot))\cap \mathcal{S}_{i,\scr O}\cap \widetilde{\mathcal{S}}_{i,A}([q_\ell(t_j)]_{\ell\in\{i+1,\dots,N\}} \cap \widetilde{\mathcal{S}}_{i,A}([q_\ell(s)]_{\ell\in\{1,\dots,i-1\}})\}$ are nonempty, $\forall i\in\{2,\dots,N\}$.
\end{assumption}

Next, similarly to the leader agent $i=1$, it calculates the predicted states $\hat{q}_i(s,\cdot),\hat{\dot{q}}_i(s,\cdot), s\in[t_j,t_j+T_p]$, which it transmits to the agents $\{i+1,\dots,N\}$. In that way, at each time instant $t_j$, each agent $i\in\{2,\dots,N\}$ measures the other agents' states (as stated in Assumption \ref{ass:sensing_assumption}), incorporates the constraint \eqref{eq:mpc_constrained_set_1 followers} for the agents $\{i+1,\dots,N\}$, receives the predicted states $\hat{q}_\ell(s,\cdot), \hat{\dot{q}}_\ell(s,\cdot)$ from the agents $\ell\in\{2,\dots,i-1\}$ and incorporates the collision avoidance constraint \eqref{eq:mpc_constrained_set_1 followers all s} for the entire horizon. Loosely speaking, we consider that each agent $i\in\mathcal{N}$ takes into account the first state of the next agents in priority ($q_\ell(t_j),\ell\in\{i+1,\dots,N\}$), as well as the transmitted predicted variables $\hat{q}_\ell(s,\cdot), \ell\in\{1,\dots,i-1\}$ of the previous agents in priority, for collision avoidance.
Intuitively, the leader agent executes the planning for the followed trajectory of the object's center of mass (through the solution of the FHOCP \eqref{eq:mpc_minimazation}-\eqref{eq:mpc_terminal_set}), the follower agents contribute in executing this trajectory through the load sharing coefficients $c_i$ (as indicated in the coupled model \eqref{eq:coupled dynamics 2}), and the agents low in priority are responsible for collision avoidance with the agents of higher priority.
Moreover, the aforementioned equality constraints \eqref{eq:equality constr_1 followers}, \eqref{eq:equality constr_2 followers} as well as the forward dynamics \eqref{eq:mpc_minimazation followers} guarantee the compliance of all the followers with the model \eqref{eq:coupled dynamics 2}. For the followers, the cost $J_i(x_i(t_j),\hat{u}_i(\cdot))$ can be selected as any function of $x_i,u_i$, $\forall i\in\{2,\dots,N\}$.

Therefore, given the constrained FHOCP \eqref{eq:mpc_minimazation followers}-\eqref{eq:mpc_constrained_set_2 followers}, the solution of problem lies in the capability of the leader agent to produce a state trajectory that guarantess  $x_{\scr O_1}(q_1(t)) \to x_\text{des}$, by solving the FHOCP \eqref{eq:mpc_minimazation}-\eqref{eq:mpc_terminal_set}, which is discussed in Theorem \ref{th:main theorem}.

\begin{definition}  \label{definition:admissible_input_with_disturbance}
	A control input $u_1 : [t_j, t_j + T_p] \to \mathbb{R}^m$ for a state $e_1(x_1(t_j))$ is called \textit{admissible} for the FHOCP \eqref{eq:mpc_minimazation}-\eqref{eq:mpc_terminal_set} if the following hold: 
	\begin{enumerate}
		\item $u_1(\cdot)$ is piecewise continuous;
		\item $u_1(s) \in U_{1,u}, \forall s \in [t_j, t_j + T_p]$;
		\item $e_1\big(x_1(s);\ u_1(\cdot); \ x_1(t_j), e_1(x_1(t_j))\big) \in \mathcal{E}_1(\cdot), \forall \ s \in [t_j, t_j+T_p]$, and 
		\item $e_1\big(x_1(t_j + T_p); \ u_1(\cdot); \ x_1(t_j), e_1(x_1(t_j))\big) \in \mathcal{F}_1(\cdot)$.
	\end{enumerate}
\end{definition}

\begin{theorem} \label{th:main theorem}
	Suppose that: 
	\begin{enumerate}
		\item Assumption \ref{ass:feasility_assumption} - \ref{ass:follower feasibility} hold;
		\item  The FHOCP \eqref{eq:mpc_minimazation}-\eqref{eq:mpc_terminal_set} is feasible for the initial time $t = 0$;
		\item There exists an admissible control input $\kappa_1 : [t_j + T_p, t_{j+1} + T_p] \to U_1$ such that for all $e_1 \in \mathcal{F}_1(\cdot)$ and for every $s \in [t_j + T_p, t_{j+1} + T_p]$ it holds that:
		\begin{enumerate}
			\item $e_1(x_1(s)) \in \mathcal{F}_1(\cdot)$, and
			\item $\dfrac{\partial V_1}{\partial e_1} g_1(e_1(x_1(s))$, $\kappa_1(s)) + F_1(e_1(x_1(s))$, $h_1(s)) \leq 0$.
		\end{enumerate}
	\end{enumerate}
Then, the system \eqref{eq:error_dynamics}, under the control input \eqref{eq:optimal_input}, converges to the set $\mathcal{F}_1(\cdot)$ when $t \to \infty$.
\end{theorem}
\begin{proof}
	The proof of the theorem consists of two parts: firstly, recursive feasibility is established, that is, initial feasibility is shown to imply subsequent feasibility; secondly, and based on the first part, it is shown that the error state $e_1(t)$ reaches the terminal set $\mathcal{F}_1(\cdot)$. The feasibility analysis and the convergence analysis is similar with the proof of Theorem $1$ in \cite[Section IV]{alex_chris_med_2017}.
\end{proof}

\section{Simulation Results} \label{sec:simulation_results}

To demonstrate the efficiency of the proposed control protocol, we consider a simulation example with $N=3$ ground vehicles equipped with $2$ DOF manipulators, rigidly grasping an object with $n_1 = n_2 = n_3 = 4$, $n = n_1+n_2+n_3= 12$. The states of the agents are given as: $q_i = [p_{\scriptscriptstyle B_i}^\top, \alpha_i^\top]^\top \in \mathbb{R}^4$, $p_{\scriptscriptstyle B_i} = [x_{\scriptscriptstyle B_i}, y_{\scriptscriptstyle B_i}]^\top \in \mathbb{R}^2$, $\alpha_i = [\alpha_{i_1}$, $\alpha_{i_2}]^\top \in \mathbb{R}^2$, $i \in \{1,2,3\}$. The state of the object is $x_{\scriptscriptstyle O} = [p_{\scriptscriptstyle O}^\top, \phi_{\scriptscriptstyle O}]^\top \in \mathbb{R}^4$ and it is calculated though the states of the agents. The manipulators become singular when $\sin(\alpha_{i_1}) = 0 \}, i \in \{1,2\}$, thus the state constraints for the manipulators are set to: $\varepsilon < \alpha_{1_1} < \frac{\pi}{2}-\varepsilon$, $-\frac{\pi}{2}+\varepsilon < \alpha_{1_2} < \frac{\pi}{2}-\varepsilon$, $-\frac{\pi}{2} + \varepsilon < \alpha_{2_1} < -\varepsilon$, $-\frac{\pi}{2}+\varepsilon < \alpha_{2_2} < \frac{\pi}{2}-\varepsilon$. We also consider the input constraints: $-8.5 \le u_{i,j}(t) \le 8.5$, $i \in \{1,2\}$, $j \in \{1, \dots, 4\}$. The initial conditions of agents and the object are set to: $q_{1}(0) = [0.5, 0, \frac{\pi}{4}, \frac{\pi}{4}]^\top$, $q_{2}(0) = [0, -4.4142, -\frac{\pi}{4}, -\frac{\pi}{4}]^\top$, $q_{3}(0) = [-0.50, -4.4142, -\frac{\pi}{4}, -\frac{\pi}{4}]^\top$, $\dot{q}_1(0) = \dot{q}_2(0) = \dot{q}_3(0) = [0, 0, 0, 0]^\top$ and $x_{\scriptscriptstyle O}(0) = [0, -2.2071, 0.9071, \frac{\pi}{2}]^\top$. The desired goal state the object is set to $x_{\scriptscriptstyle O, \text{des}} = [5, -2.2071, 0.9071, \frac{\pi}{2}]^\top$, which, due to the structure of the considered robots, corresponding uniquely to $q_{1, \text{des}} = [5.5, 0, \frac{\pi}{4}, \frac{\pi}{4}]^\top$, $q_{2, \text{des}} = [5, -4.4142, -\frac{\pi}{4}, -\frac{\pi}{4}]^\top$, $q_{3, \text{des}} = [4.5, 0, -\frac{\pi}{4}, -\frac{\pi}{4}]^\top$, $\dot{q}_{3, \text{des}} = [0, 0, 0, 0]^\top$ and $\dot{q}_{1, \text{des}} = \dot{q}_{2, \text{des}} = \dot{q}_{3, \text{des}} =  [0, 0, 0, 0]^\top$. We set an obstacle between the initial and the desired pose of the object. The obstacle is spherical with center $[2.5,-2.2071,1]$ and radius $\sqrt{0.2}$. The sampling time is $h = 0.1 \sec$, the horizon is $T_p = 0.5 \sec$, and the total simulation time is $60 \sec$; The matrices $P$, $Q$, $R$ are set to: $P = Q = 0.5 I_{8 \times 8}$, $R = 0.5 I_{4 \times 4}$ and the load sharing coefficients as $c_1 = 0.3$, $c_2 = 0.5$, and $c_3 = 0.2$. The simulation results are depicted in Fig. \ref{fig:error_ag_1}- Fig. \ref{fig:obstacle_function}; Fig. \ref{fig:error_ag_1}, Fig. \ref{fig:errors_ag_2} and Fig. \ref{fig:errors_ag_3} show the error states of agent $1$, $2$ and $3$, respectively, which converge to $0$; Fig. \ref{fig:states_object} depicts the states of the objects; Fig. \ref{fig:obstacle_function} shows the collision-avoidance constraint with the obstacle; Fig. \ref{fig:control_inputs_ag_1} - Fig. \ref{fig:control_inputs_ag_2} depict the control inputs of the three agents. Note that the different load-sharing coefficients produce slightly different inputs. The simulation was carried out by using the NMPC toolbox given in \cite{grune_2011_nonlinear_mpc} and it took $13450 \sec$ in MATLAB $R2015$a Environment on a desktop with $8$ cores, $3.60$ GHz CPU and $16$GB of RAM. In our previous work \cite{alex_chris_med_2017}, the same simulation was \emph{centralized} and it took $45547 \sec$ on the same computer.
\begin{figure}[t!]
	\vspace{2mm}
	\centering
	\includegraphics[scale = 0.50]{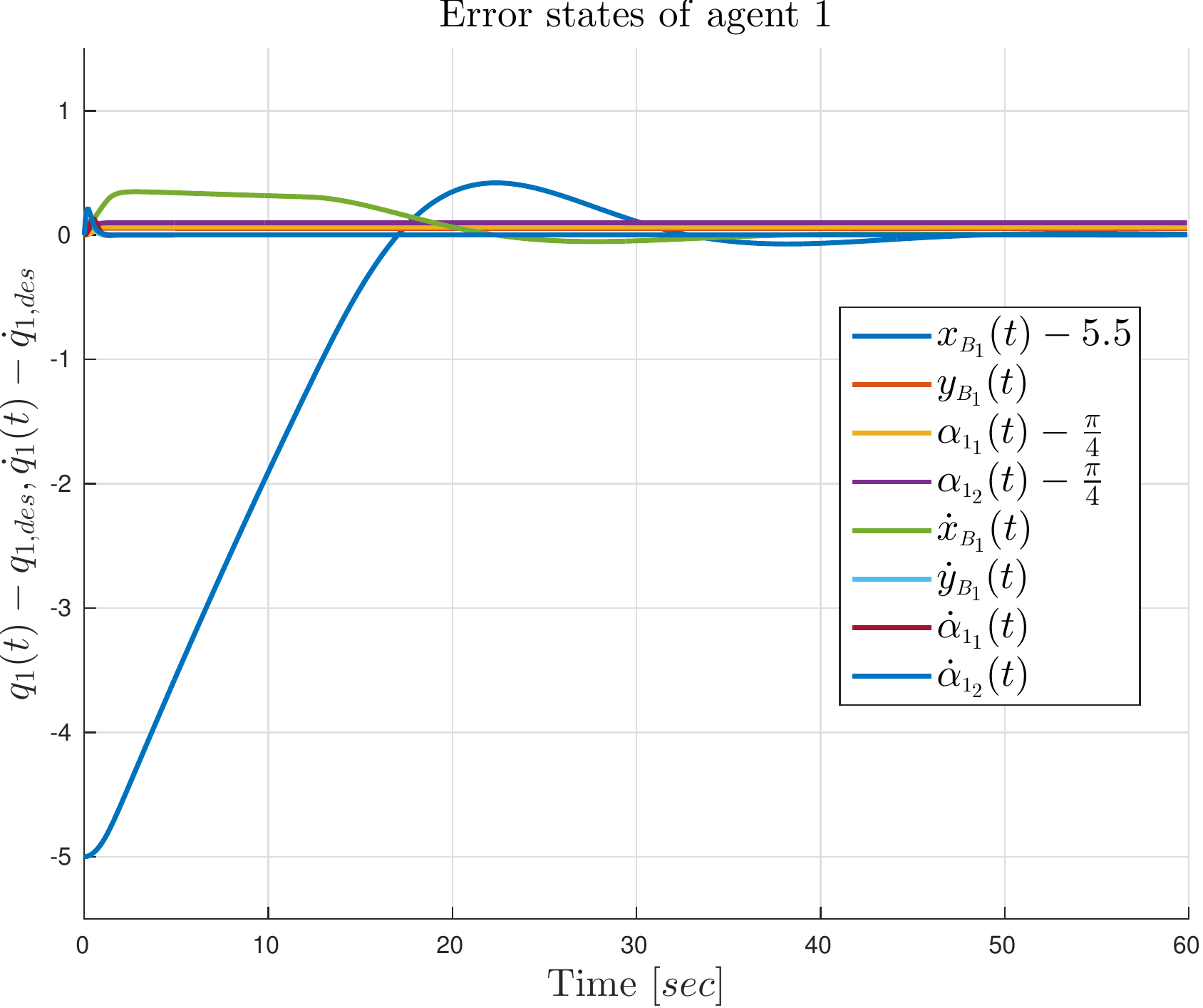}
	\caption{The error states of agent $1$.}
	\label{fig:error_ag_1}
\end{figure}
\begin{figure}[t!]
	\centering
	\includegraphics[scale = 0.50]{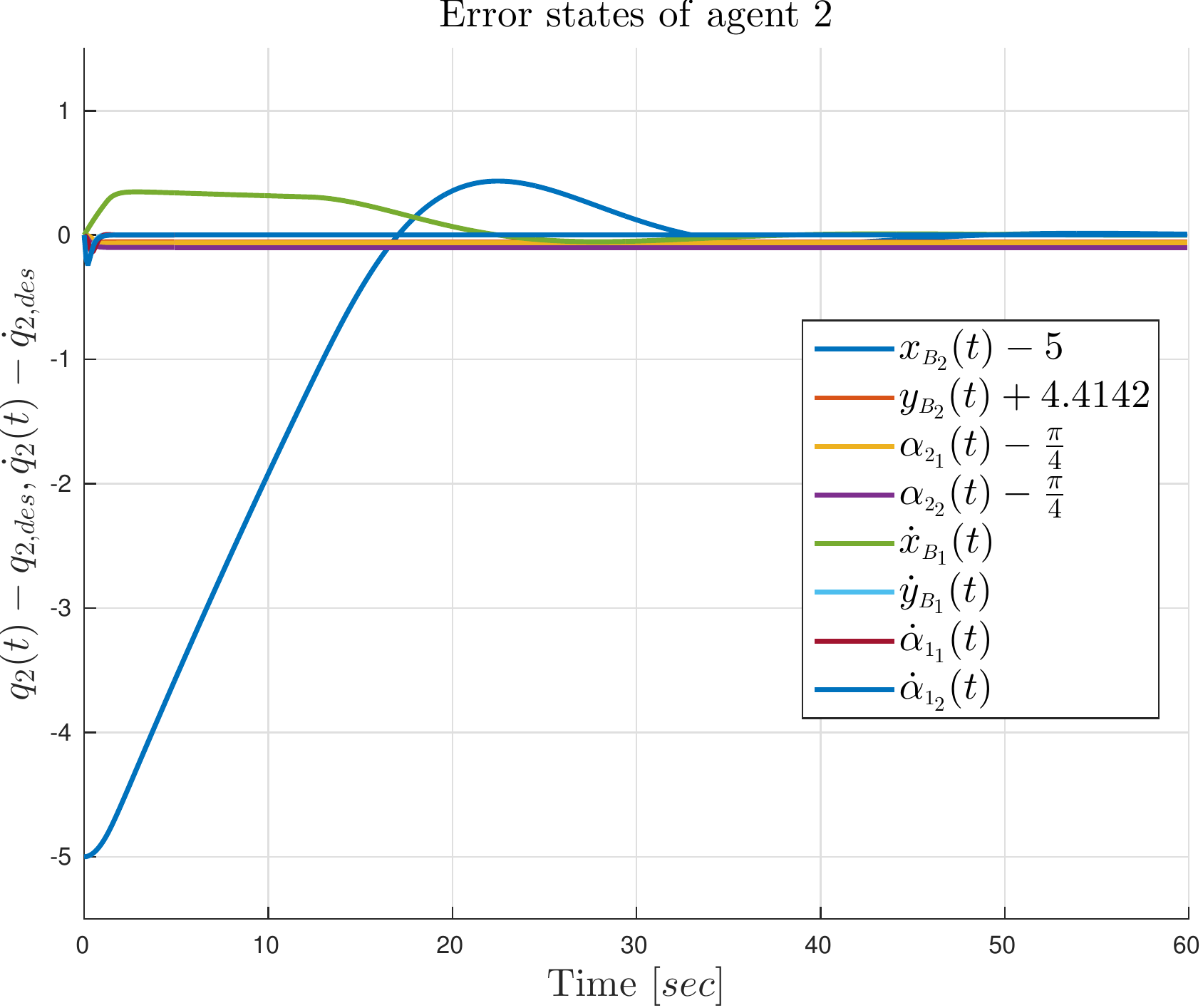}
	\caption{The error states of agent $2$.}
	\label{fig:errors_ag_2}
\end{figure}

\begin{figure}[t!]
	\vspace{2mm}
	\centering
	\includegraphics[scale = 0.50]{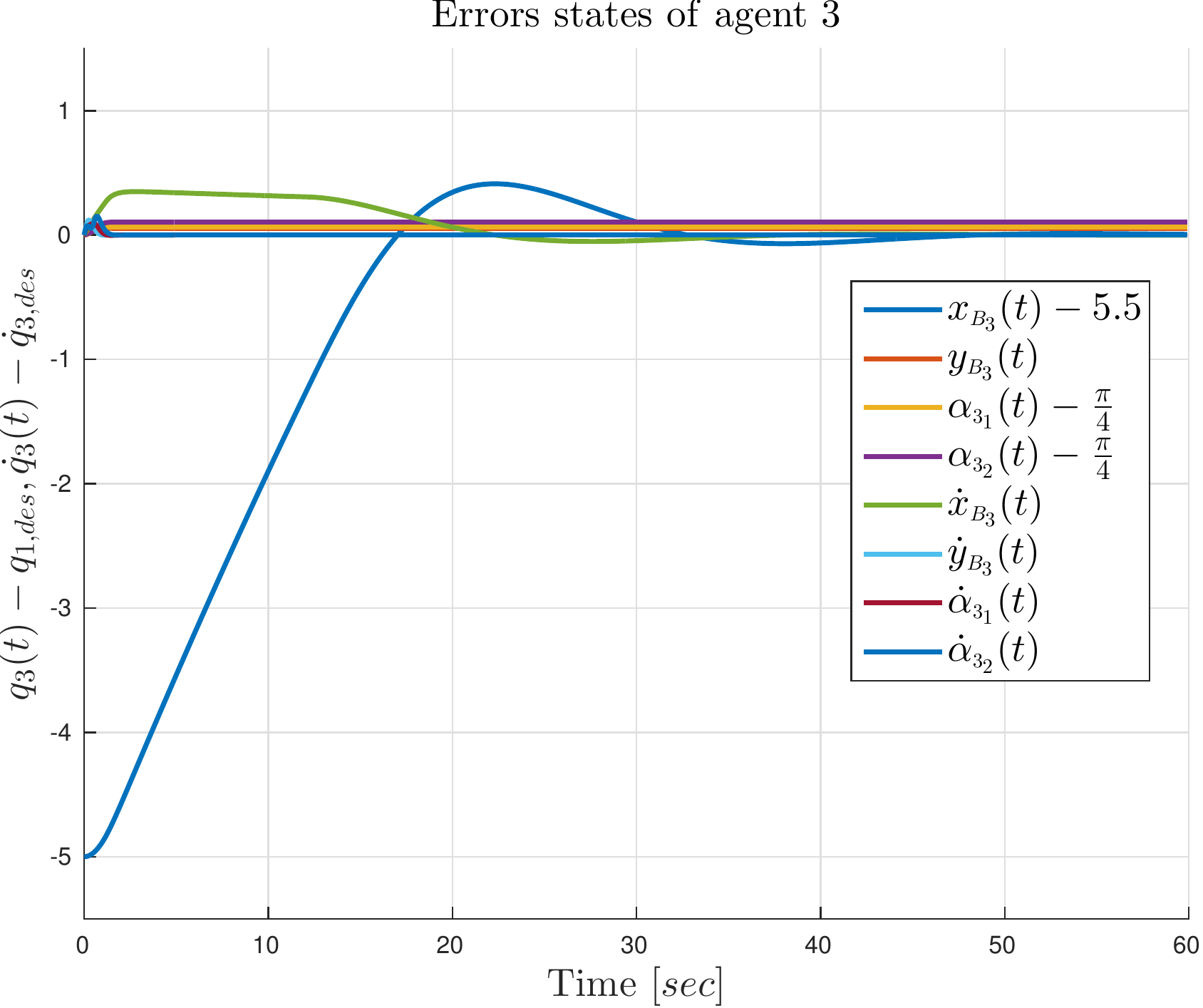}
	\caption{The error states of agent $3$.}
	\label{fig:errors_ag_3}
\end{figure}

\begin{figure}[t!]
	\centering
	\includegraphics[scale = 0.50]{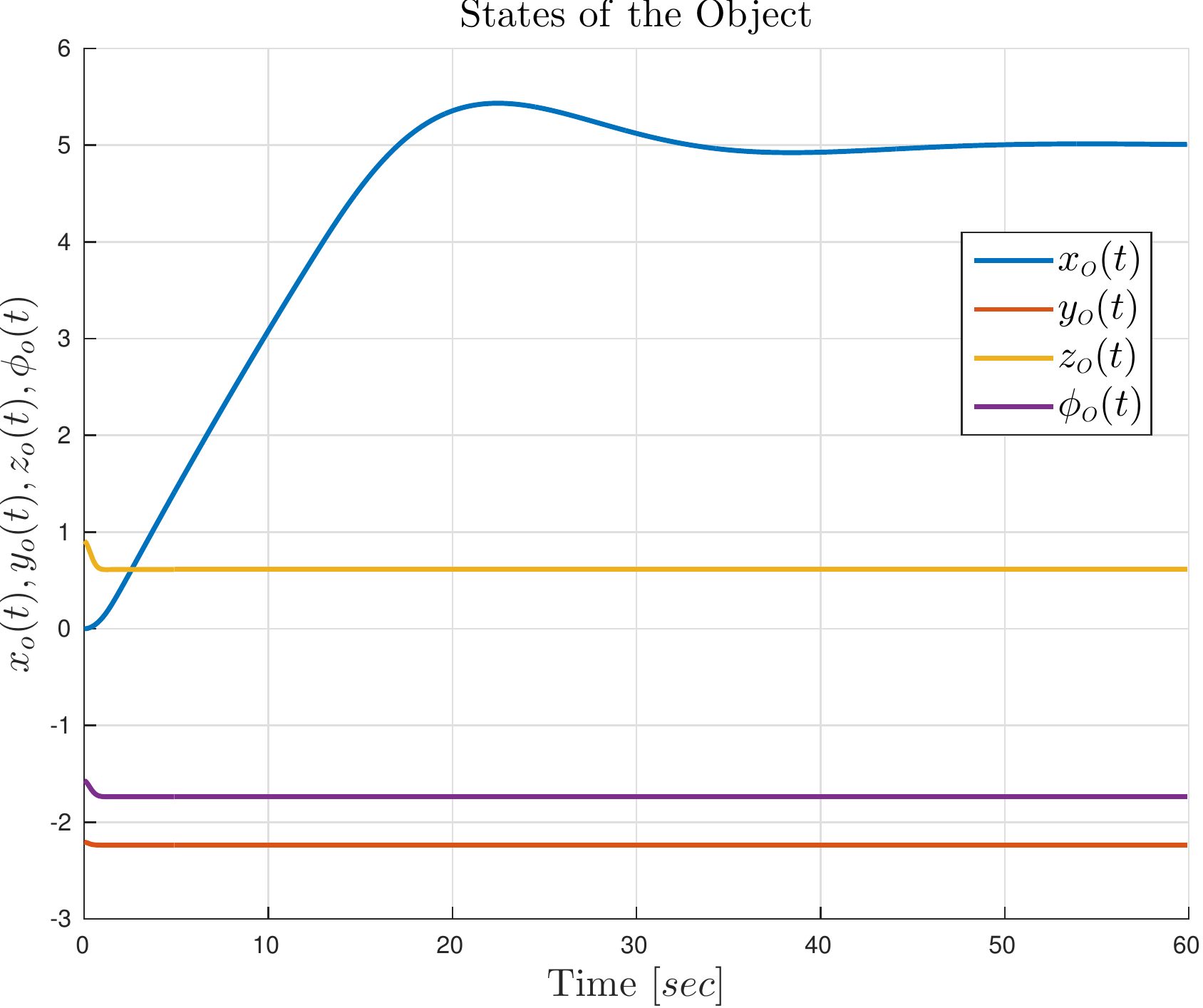}
	\caption{The states of object converging to the desired ones.}
	\label{fig:states_object}
\end{figure}

\begin{figure}[t!]
	\centering
	\includegraphics[scale = 0.50]{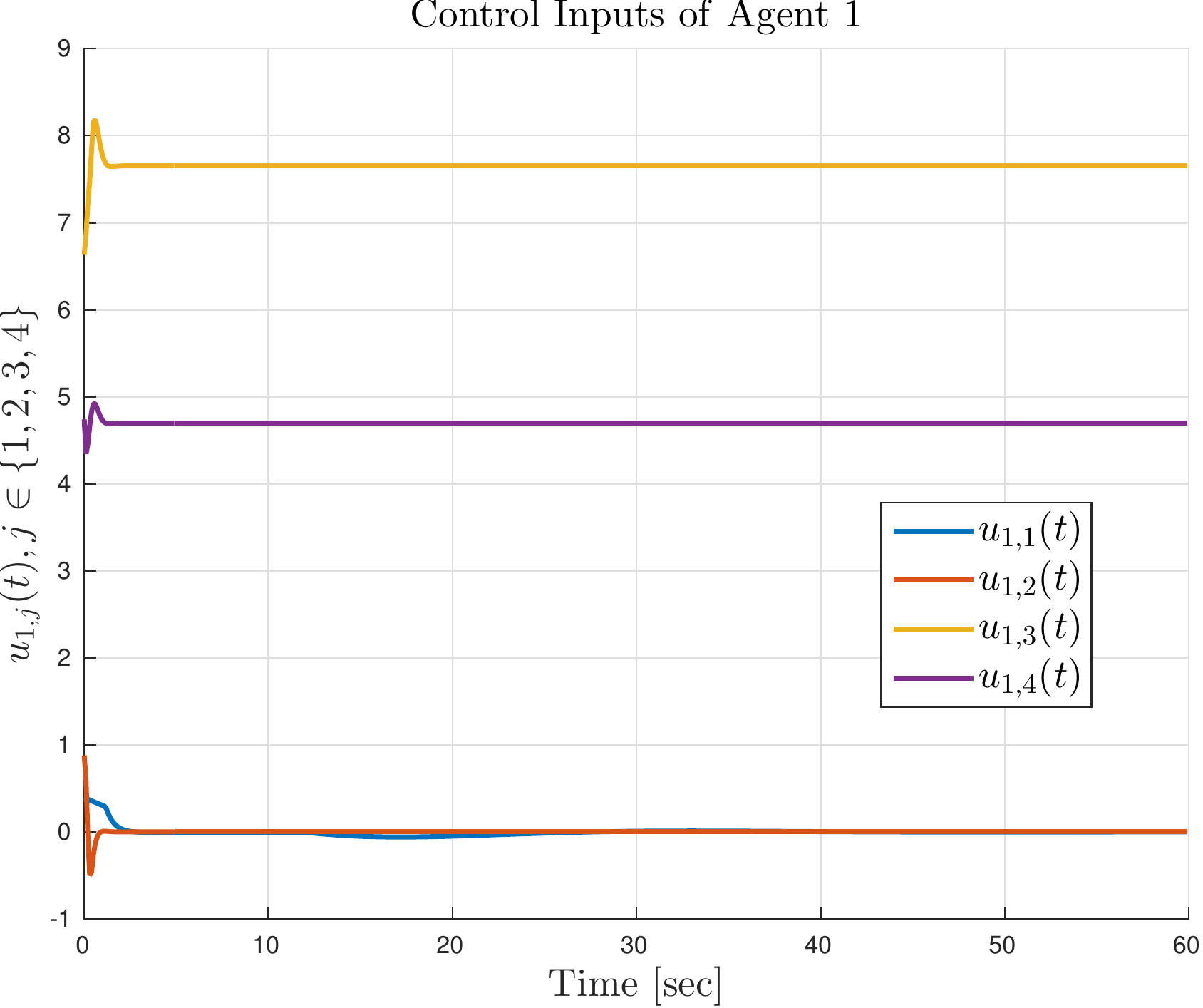}
	\caption{The control inputs of agent $1$ with $-8.5 \le u_{1,j}(t) \le 8.5$.}
	\label{fig:control_inputs_ag_1}
\end{figure}

\begin{figure}[t!]
	\centering
	\includegraphics[scale = 0.50]{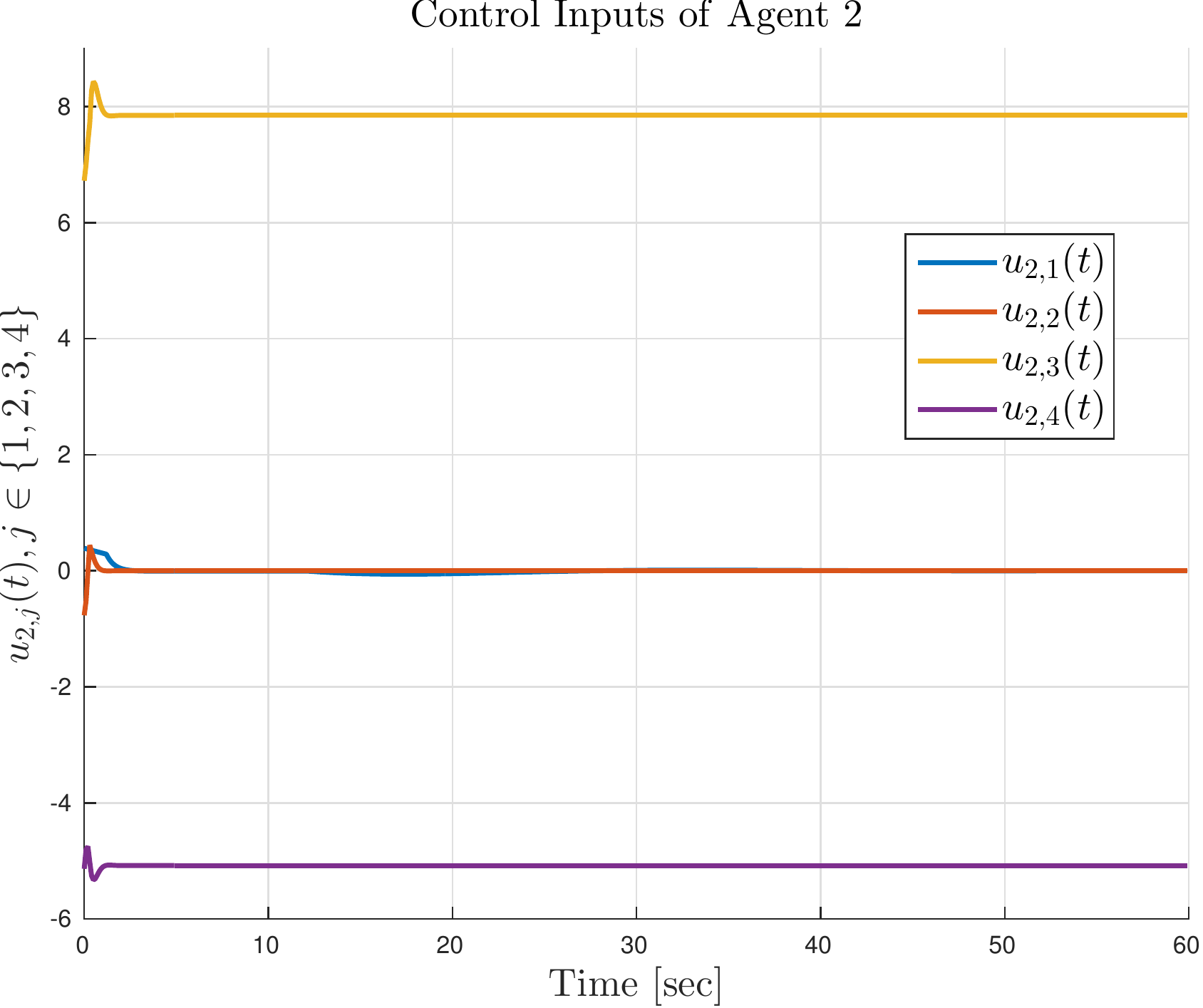}
	\caption{The control inputs of agent $2$ with $-8.5 \le u_{2,j}(t) \le 8.5$.}
	\label{fig:control_inputs_ag_3}
\end{figure}

\begin{figure}[t!]
	\centering
	\includegraphics[scale = 0.50]{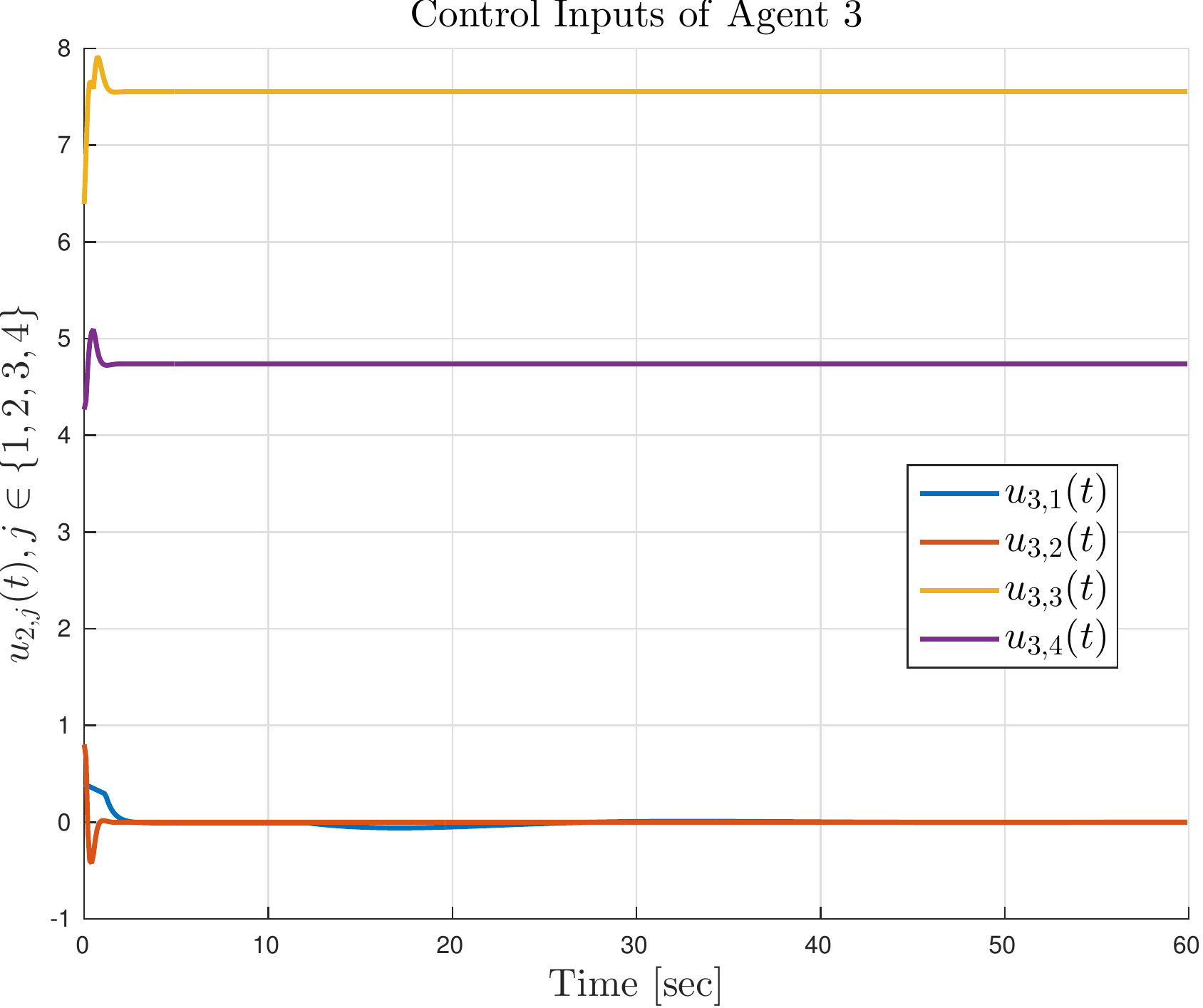}
	\caption{The control inputs of agent $3$ with $-8.5 \le u_{3,j}(t) \le 8.5$.}
	\label{fig:control_inputs_ag_2}
\end{figure}

\begin{figure}[t!]
	\vspace{2mm}
	\centering
	\includegraphics[scale = 0.50]{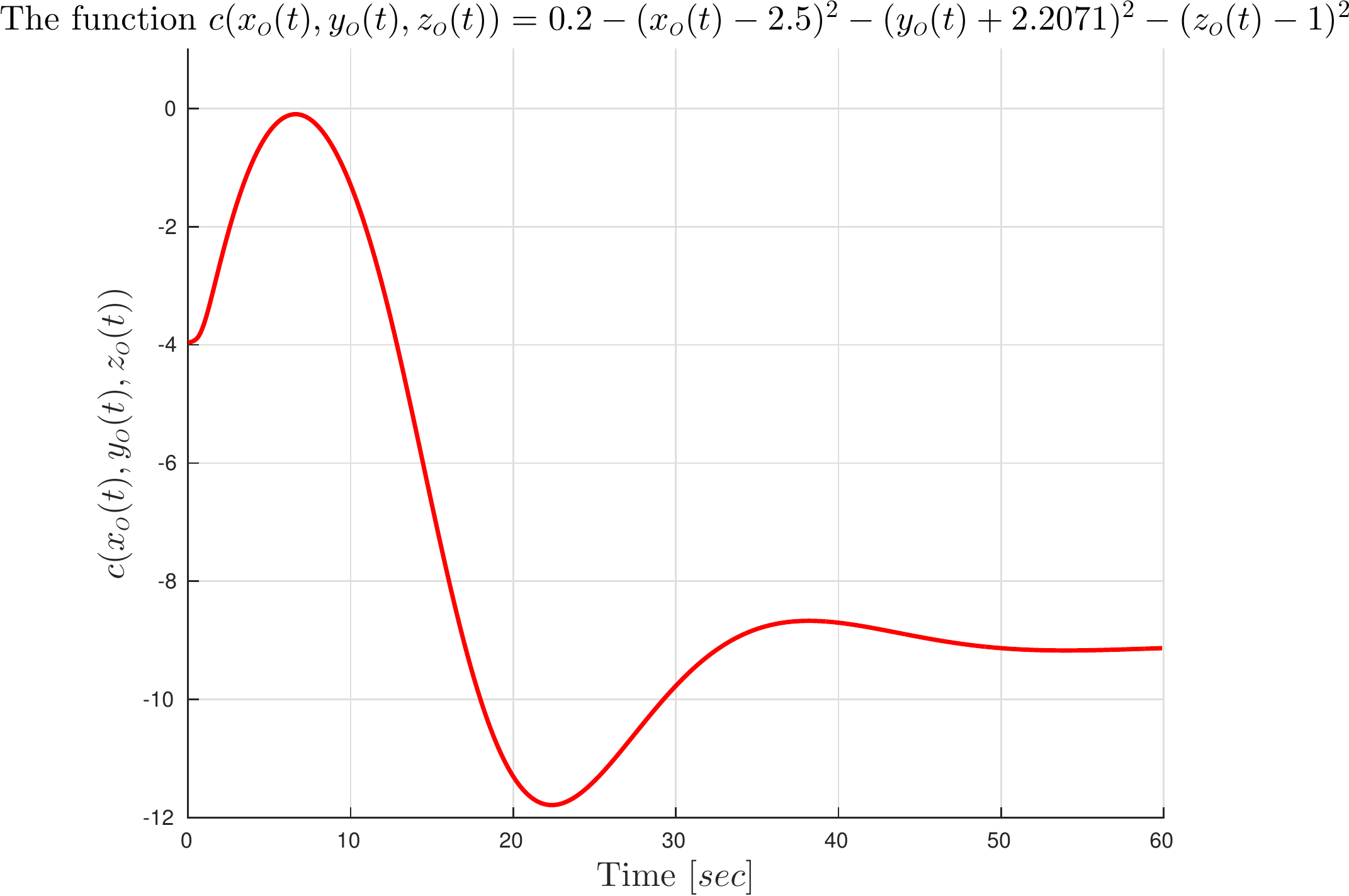}
	\caption{The function $c(x_{\scriptscriptstyle O}(t)$, $y_{\scriptscriptstyle O}(t)$, $z_{\scriptscriptstyle O}(t)) = 0.2-(x_{\scriptscriptstyle O}-2.5)^2-(y_{\scriptscriptstyle O}+2.2071)^2-(z_{\scriptscriptstyle O}-1)^2$ is always negative, indicating collision avoidance.}
	\label{fig:obstacle_function}
\end{figure}

\section{Conclusions and Future Work} \label{sec:conclusions}

In this work we proposed a NMPC scheme for decentralized cooperative transportation of an object rigidly grasped by $N$ robotic agents. The proposed control scheme deals with singularities of the agents, inter-agent collision avoidance as well as collision avoidance between the agents and the object with the workspace obstacles. We proved the feasibility and  convergence analysis of the proposed methodology and simulation results verified the efficiency of the approach. Future efforts will be devoted towards reconfiguration in case of task infeasibility for the followers,  event-triggered communication between the agents so as to reduce the communication burden that is required for solving the FHOCP at every sampling time, and real-time experiments.

\bibliography{references}
\bibliographystyle{ieeetr}
\end{document}